\documentclass[9pt,twoside]{extarticle}
\usepackage{microtype}
\usepackage{graphicx}
\usepackage{subfigure}
\usepackage{epstopdf}
\usepackage{booktabs} % for professional tables
\usepackage{enumitem}
\usepackage{amsmath, amsthm, amssymb, amsfonts}
\usepackage{natbib}
\usepackage{mathtools}
\usepackage{authblk}
\usepackage{times}
\usepackage[papersize={17cm,24cm},margin=22mm,top=17mm,headsep=5mm]{geometry} 
\usepackage[dvipsnames]{xcolor}
\usepackage[english]{babel} % multilinguality
\date{}
\renewenvironment{abstract}
  {\noindent\small %\quotation
  {\noindent{\large\textbf\abstractname. }%\par\nobreak\smallskip
  \thispagestyle{plain}
  }}
  {}
\usepackage{fancyhdr}
\renewcommand{\headrulewidth}{0.4pt}
\fancypagestyle{plain}{% first page of chapter
  \fancyhead[C,R]{}
  \renewcommand{\headrulewidth}{0pt} % comment to keep rule
  \fancyfoot{} % comment to put page number at bottom
  }
\usepackage[colorlinks,linkcolor = Aquamarine, urlcolor  = violet, citecolor = YellowOrange, anchorcolor = violet]{hyperref}
\usepackage[linesnumbered,ruled]{algorithm2e}
\usepackage{algorithmic}
\newcommand{\argmin}{\operatornamewithlimits{argmin}}
\newtheorem{theorem}{Theorem}[section]
\newtheorem{lemma}[theorem]{Lemma}

\newtheorem{corollary}[theorem]{Corollary}
\newcommand*\E[1]{\mathbb{E}\left[#1\right]}
\newcommand*\Ep[2]{\mathbb{E}_{#1}\left[#2\right]}

\renewcommand*\d[0]{\text{d}}

\newcommand\numberthis{\addtocounter{equation}{1}\tag{\theequation}}

\newcommand*\eu[1]{\left\| #1\right\|}

\renewcommand*\d[0]{\delta}
\newcommand*\xt[0]{\tilde{x}}
\newcommand*\vt[0]{\tilde{v}}
\newcommand*\pt[0]{\tilde{p}}

\newcommand*\Phit[0]{\tilde{\Phi}}

\def\lv{\lVert}
\def\rv{\rVert}
\def\ke{\mathcal{E}_K}
\DeclarePairedDelimiter\floor{\lfloor}{\rfloor}

\begin{document}

\title{On the Theory of Variance Reduction for Stochastic Gradient Monte Carlo}

%\author{Niladri Chatterji%\thanks{University of California, Berkeley}
%\and Nicolas Flammarion\footnotemark[1]
%\and Yi-An Ma\footnotemark[1]
%\and Peter L. Bartlett\footnotemark[1]
%\and Michael I. Jordan\footnotemark[1]}
%\author[1]{Niladri Chatterji}
%\author[1]{Nicolas Flammarion}
%\author[1]{Yi-An Ma}
%\author[1]{Peter L. Bartlett}
%\author[1]{Michael I. Jordan}
%\affil[1]{University of California, Berkeley}
\author{Niladri Chatterji\thanks{niladri.chatterji@berkeley.edu}}
\author{Nicolas Flammarion\thanks{flammarion@berkeley.edu}}
\author{Yi-An Ma\thanks{yianma@berkeley.edu}}
\author{Peter Bartlett\thanks{peter@berkeley.edu}}
\author{Michael Jordan\thanks{jordan@cs.berkeley.edu}}
\affil{University of California, Berkeley}

\maketitle

% this must go after the closing bracket ] following \twocolumn[ ...

% This command actually creates the footnote in the first column
% listing the affiliations and the copyright notice.
% The command takes one argument, which is text to display at the start of the footnote.
% The \icmlEqualContribution command is standard text for equal contribution.
% Remove it (just {}) if you do not need this facility.

%\printAffiliationsAndNotice{}  % leave blank if no need to mention equal contribution
%\printAffiliationsAndNotice{\icmlEqualContribution} % otherwise use the standard text.

\begin{abstract}
We provide convergence guarantees in Wasserstein distance for a variety of variance-reduction methods: SAGA Langevin diffusion, SVRG Langevin diffusion and control-variate underdamped Langevin diffusion. We analyze these methods under a uniform set of assumptions on the log-posterior distribution, assuming it to be smooth, strongly convex and Hessian Lipschitz. This is achieved by a new proof technique combining ideas from finite-sum optimization and the analysis of sampling methods. Our sharp theoretical bounds allow us to identify regimes of interest where each method performs better than the others. Our theory is verified with experiments on real-world and synthetic datasets.

\end{abstract}

\section{Introduction}
One of the major themes in machine learning is the use of stochasticity
to obtain procedures that are computationally efficient and statistically
calibrated.  There are two very different ways in which this theme has
played out---one frequentist and one Bayesian.  On the frequentist side,
gradient-based optimization procedures are widely used to obtain point estimates and
point predictions, and stochasticity is used to bring down the computational
cost by replacing expensive full-gradient computations with unbiased
stochastic-gradient computations.  On the Bayesian side, posterior
distributions provide information about uncertainty in estimates and
predictions, and stochasticity is used to represent those distributions
in the form of Monte Carlo (MC) samples.  Despite the different conceptual
frameworks, there are overlapping methodological issues.  In particular,
Monte Carlo sampling must move from an out-of-equilibrium configuration
towards the posterior distribution and must do so quickly, and thus
optimization ideas are relevant.  Frequentist inference often involves
sampling and resampling, so that efficient approaches to Monte Carlo
sampling are relevant.

Variance control has been a particularly interesting point of contact
between the two frameworks.  In particular, there is a subtlety in the
use of stochastic gradients for optimization: Although the per-iteration
cost is significantly lower by using stochastic gradients; extra variance
is introduced into the sampling procedure at every step so that the total
number of iterations is required to be larger.  A natural question is
whether there is a theoretically-sound way to manage this tradeoff.
This question has been answered affirmatively in a seminal line of
research~\citep{schmidt2017minimizing, ShaZha13,zhanjon13} on
variance-controlled stochastic optimization. Theoretically these
methods enjoy the best of the gradient and stochastic gradient
worlds---they converge at the fast rate of full gradient methods
while making use of cheaply-computed stochastic gradients.

A parallel line of research has ensued on the Bayesian side in a
Monte Carlo sampling framework.  In particular, stochastic-gradient
Markov chain Monte Carlo (SG-MCMC) algorithms have been proposed in
which approximations to Langevin diffusions make use of stochastic
gradients instead of full gradients~\cite{SGLD}.  There have been
a number of theoretical results that establish mixing time bounds for
such Langevin-based sampling methods when the posterior distribution is well
behaved \citep{Dal14,durmus2017,cheng2017convergence,dalalyan2017user}.
Such results have set the stage for the investigation of variance
control within the SG-MCMC framework~\cite{dubey2016variance,Dur16,bierkens2016zig,
baker2017control,nagapetyan2017true,chen2017convergence}.  Currently,
however, the results of these investigations are inconclusive.
\cite{dubey2016variance} obtain mixing time guarantees for
SAGA Langevin diffusion and SVRG Langevin diffusion (two particular
variance-reduced sampling methods) under the strong assumption that
the log-posterior has the norm of its gradients  uniformly
bounded by a constant.  Another approach that has been explored
involves calculating the mode of the log posterior to construct a
control variate for the gradient estimate~\citep{baker2017control,
nagapetyan2017true}, an approach that makes rather different assumptions.
Indeed, the experimental results from these two lines of work are
contradictory, reflecting the differences in assumptions.

In this work we aim to provide a unified perspective on variance
control for SG-MCMC.  Critically, we identify two regimes: we show
that when the target accuracy is small, variance-reduction methods
are effective, but when the target accuracy is not small (a low-fidelity
estimate of the posterior suffices), stochastic gradient Langevin
diffusion (SGLD) performs better.  These results are obtained
via new theoretical techniques for studying stochastic gradient
MC algorithms with variance reduction.  We improve upon the
techniques used to analyze Langevin Diffusion (LD) and SGLD
\cite{Dal14,dalalyan2017user,durmus2017} to establish
non-asymptotic rates of convergence (in Wasserstein distance) for
variance-reduced methods. We also apply control-variate techniques
to underdamped Langevin MCMC \citep{cheng2017underdamped}, a second-order
diffusion process (CV-ULD). Inspired by proof techniques for
variance-reduction methods for stochastic optimization, we design a Lyapunov
function to track the progress of convergence and we thereby obtain
better bounds on the convergence rate. We make the relatively weak
assumption that the log posteriors are Lipschitz smooth, strongly
convex and Hessian Lipschitz---a relaxation of the strong assumption
that the gradient of the log posteriors are globally bounded.

As an example of our results, we are able to show that when using a
variance-reduction method $\tilde{\mathcal{O}}(N+ \sqrt{d}/\epsilon)$
steps are required to obtain an accuracy of $\epsilon$, versus the
$\tilde{\mathcal{O}}(d/\epsilon^2)$ iterations required for SGLD, where
$d$ is the dimension of the data and $N$ is the total number of samples.
As we will argue, results of this kind support our convention that
when the target accuracy $\epsilon$ is \emph{small}, variance-reduction
methods outperform SGLD.
\subsection*{Main Contributions}
We provide sharp convergence guarantees
for a variety of variance-reduction methods---SAGA-LD, SVRG-LD, and
CV-ULD under the \emph{same set} of \emph{realistic} assumptions (see
Sec.~\ref{sec:convergenceresults}). This is achieved by a new
proof technique that yiels bounds on Wasserstein distance.
Our bounds allow us to identify windows of interest where each
method performs better than the others (see Fig.~\ref{fig:regime}).
The theory is verified with experiments on real-world datasets.
We also test the effects of breaking the central limit theorem
using synthetic data, and find that in this regime variance-reduced
methods fare far better than SGLD (see Sec.~\ref{sec:experiments}).

\section{Preliminaries}
%!TEX root = main.tex
Throughout the paper we aim to make inference on a vector of parameters $\theta \in \mathbb{R}^d$. The resulting posterior density is then $p(\theta \lvert \mathbf{z}) \propto p(\theta) \prod_{i=1}^N p(z_i \lvert \theta)$. For brevity we write $f_i(\theta) = -\log(p(z_i \lvert \theta))$, for $i\in \{1,\ldots,N\}$, $f_0(x) = -\log(p(\theta))$ and $f(\theta) = -\log(p(\theta\lvert \mathbf{z}))$. Moving forward we state all results in terms of general sum-decomposable functions $f$ (see Assumption \ref{ass:decom}), however it is useful to keep the above example in mind as the main motivating example. We let $\lVert v \rVert_2$ denote the Euclidean norm, for a vector $v \in \mathbb{R}^d$.  For a matrix $A$ we let $\lVert A \rVert$ denote its spectral norm and let $\lv A \rv_{F}$ denote its Frobenius norm.
%Finally we will use $g(N) = \mathcal{O}(f(N))$ to signify that $g(N)$ is upper bounded by $Cf(N)$ for some small enough constant $C$.

\textbf{Assumptions on $f$}: We make the following assumptions about the potential function $f : \mathbb{R}^d \mapsto \mathbb{R}$. \vspace{-10pt}
\begin{enumerate}[label=(A{\arabic*})]
    \item \label{ass:decom} \textbf{Sum-decomposable:} The function $f$ is decomposable, $f(x) = \sum_{i=1}^{N} f_i(x)$.
    \item \label{ass:smoothness} \textbf{Smoothness:} The functions $f_i$ are twice continuously-differentiable on $\mathbb{R}^d$ and have Lipschitz-continuous gradients; that is, there exist positive constants $\tilde{M} >0$ such that for all $x,y \in \mathbb{R}^d$ and for all $i \in \{1,\ldots,N\}$ we have, $ \lVert \nabla f_i(x) - \nabla f_i(y) \rVert_2 \le \tilde{M} \lVert x - y \rVert_2. $ We accordingly characterize the smoothness of $f$ with the parameter $M: = N \tilde{M}$.
   \item \label{ass:strongconvexity}\textbf{Strong Convexity:} $f$ is $m$-strongly convex; that is, there exists a constant $m>0$ such that for all $x,y \in \mathbb{R}^d$, $
   f(y) \ge f(x) + \langle \nabla f(x),y-x \rangle + \frac{m}{2}\lVert x-y \rVert_2^2.
   $
    We also define the condition number $\kappa := M/m$.
   \item \label{ass:HessianLipschitz}\textbf{Hessian Lipschitz:} We assume that the function $f$ is Hessian Lipschitz; that is, there exists a constant $L>0$ such that, $\lVert \nabla^2 f(x) - \nabla^2 f(y) \rVert \le L \lVert x - y \rVert_2$ for every $x,y \in \mathbb{R}^d$.
\end{enumerate}
\vspace{-7pt}
It is worth noting that $M,m$ and $L$ can all scale with $N$.

\textbf{Wasserstein Distance:} We define the Wasserstein distance between a pair of probability measures ($\mu,\nu$) as follows:
\vspace{-6pt}
\begin{align*}
W^2_2(\mu,\nu) :=\inf_{\zeta\in\Gamma(\mu,\nu)} \int \lVert x-y\rVert_2^2 d\zeta(x,y),
\vspace{-6pt}
\end{align*}
where $\Gamma(\mu,\nu)$ denotes the set of joint distributions such that the first set of coordinates has marginal $\mu$ and the second set has marginal $\nu$. (See Appendix \ref{app:wass} for a more formal definition of $W_2$).

\textbf{Langevin Diffusion:} The classical overdamped Langevin diffusion is based on the following It\^{o} Stochastic Differential Equation (SDE):
\vspace{-6pt}
\begin{align}\label{eq:contlangevindynamics}
d x_t = -\nabla f(x_t) dt + \sqrt{2}dB_t,
\vspace{-6pt}
\end{align}
where $x_t \in \mathbb{R}^d$ and $B_t$ represents standard Brownian motion \citep[see, e.g.,][]{morters2010brownian}. It can be shown that under mild conditions like $\exp(-f(x)) \in L^{1}$ (absolutely integrable) the invariant distribution of Eq.~\eqref{eq:contlangevindynamics} is given by $p^*(x) \propto \exp(-f(x))$. This fact motivates the Langevin MCMC algorithm where given access to full gradients it is possible to efficiently simulate the discretization,
\vspace{-6pt}
\begin{align} \label{eq:disclangevin1}
d\tilde{x}_t = -\nabla f(x_k)dt + \sqrt{2}dB_t,
\vspace{-6pt}
\end{align}
where the gradient is evaluated at a fixed point $x_k$ (the previous iterate in the chain) and the SDE \eqref{eq:disclangevin1} is integrated up to time $\delta$ (the step size) to obtain $$x_{k+1} = x_k - \delta\nabla f(x_k) + \sqrt{2\delta}\xi_k,$$ with $\xi_k \sim N(0,I_{d\times d})$. \cite{SGLD} proposed an alternative algorithm---Stochastic Gradient Langevin Diffusion (SGLD)---for sampling from sum-decomposable function where the chain is updated by integrating the SDE:
\vspace{-6pt}
\begin{align}\label{eq:SGLD}
d\tilde{x}_t = - g_k dt + \sqrt{2} dB_t,
\vspace{-6pt}
\end{align}
and where $g_k =\frac{N}{n} \sum_{i \in S} \nabla f_i(x_k)$ is an unbiased estimate of the gradient at $x_k$. The attractive property of this algorithm is that it is computationally tractable for large datasets (when $N$ is large). At a high level the variance reduction schemes that we study in this paper replace the simple gradient estimate in Eq.~\eqref{eq:SGLD} (and other variants of Langevin MCMC) with more sophisticated unbiased estimates that have lower variance.

\section{Variance Reduction Techniques}
%!TEX root = main.tex
In the seminal work of \cite{schmidt2017minimizing} and \cite{zhanjon13}, it was observed that the variance of Stochastic Gradient Descent (SGD) when applied to optimizing sum-decomposable strongly convex functions decreases to zero only if the step-size also decays at a suitable rate. This prevents the algorithm from converging at a linear rate, as opposed to methods like batch gradient descent that use the entire gradient at each step. They introduced and analyzed different gradient estimates with lower variance. Subsequently these methods were also adapted to Monte Carlo sampling by \cite{dubey2016variance,nagapetyan2017true,baker2017control}. These methods use information from previous iterates and are no longer Markovian. In this section we describe several variants of these methods. %A similar story is also true for sampling methods and in this section we will first introduce two variants of these methods classical in machine learning: SAGA and SVRG \textcolor{red}{cite appropriate papers here, give full form. Also mention the papers that study this in the context of sampling here again?}.
\subsection{SAGA Langevin MC}
We present a sampling algorithm based on SAGA of \cite{DefBacLac14} which was developed as a modification of SAG  by \cite{schmidt2017minimizing}. In SAGA, which is presented as Algorithm \ref{alg:saga}, an approximation of the gradient of each function $f_i$ is stored as $\{g_k^i\}_{i=1}^N$ and is iteratively updated in order to build an estimate with reduced variance. At each step of the algorithm, if the function $f_i$ is selected in the mini-batch $S$, then the value of the gradient approximation is updated by setting $g^i_{k+1}=\nabla f_i(x_k)$. Otherwise the gradient of $f_i$ is approximated by the previous value $g_k^i$. Overall we obtain the following unbiased estimate of the gradient:
\vspace{-5pt}
\begin{equation}\label{eq:sagagradientupdate}
g_k=\sum_{i=1}^n g_k^i+\frac{N}{n}\sum_{i\in S} (\nabla f_i(x_k)-g_k^i).
\vspace{-5pt}
\end{equation}
In Algorithm \ref{alg:saga} we form this gradient estimate and plug it into the classic Langevin MCMC method driven by the SDE \eqref{eq:SGLD}. %by updating the next iterate $x_{k+1}$ as an appropriate combination the previous iterate, the gradient estimate and a Gaussian random variable.
Computationally this algorithm is efficient; essentially it enjoys the oracle query complexity (number of calls to the gradient oracle per iteration) of methods like SGLD but due to the reduced variance of the gradient estimator it converges almost as quickly (in terms of number of iterations) to the posterior distribution as methods such as Langevin MCMC that use the complete gradient at every step. We prove a novel non-asymptotic convergence result in Wasserstein distance for Algorithm \ref{alg:saga} in the next section that formalizes this intuition.

The principal downside of this method is its memory requirement. It is necessary to store the gradient estimator for each individual $f_i$, which essentially means that in the worst case the memory complexity scales as $\mathcal{O}(Nd)$. However in many interesting applications, including some of those considered in the experiments in Sec.~\ref{sec:exp}, the memory costs scale only as $\mathcal{O}(N)$ since each function $f_i$ depends on a linear function in $x$ and therefore the gradient $\nabla f_i$ is just a re-weighting of the single data point $z_i$.
%For many problems each derivative f
%0
%i
%is just a simple weighting of the ith data vector.
%Logistic regression and least squares have this property. In such cases, instead of storing
%the full derivative f
%0
%i
%for each i, we need only to store the weighting constants. This reduces
%the storage requirements to be the same as the SDCA method in practice. A similar trick
%can be applied to multi-class classifiers with p classes by storing p − 1 values for each i.\textcolor{red}{why is this true? Any reference we can put for this?}. \note{comment about the minibatch}
\label{sec:saga}
\setlength{\textfloatsep}{10pt}
\begin{algorithm}[t]
\caption{SAGA Langevin MCMC \label{alg:saga}}
\begin{algorithmic}
   \STATE {\bfseries Input:} Gradient oracles $\{\nabla f_i(\cdot)\}_{i=0}^N$, step size $\delta$, batch size $n$, initial point $x_0 \in \mathbb{R}^d$.
   \STATE Initialize $\{g_0^i=\nabla f_i(x_0)\}_{i=1}^N$.
   \FOR {$k=1,\ldots,T$}
   \STATE Draw $S \subset \{0,\ldots,N\}: \lvert S \rvert = n$ uniformly with replacement
   \STATE Sample $\xi_k \sim N(0, I_{d\times d})$
   \STATE Update $g_k$ using \eqref{eq:sagagradientupdate}
      \STATE Update $x_{k+1}\leftarrow x_k - \delta g_k + \sqrt{2\delta}\xi_k$.
   \STATE Update  $\{g_k^i\}_{i=1}^N$:   for $i\in S$ set  $g^i_{k+1}=\nabla f_i(x_k)$, for $i \in S^{c}$, set  $g^i_{k+1}=g^i_{k}$
   \ENDFOR
   \STATE{\bfseries Output:} Iterates $\{x_k\}_{k=1}^T$.
\end{algorithmic}
\end{algorithm}

\subsection{SVRG Langevin MC}
%!TEX root = main.tex
The next algorithm we explore is based on the SVRG method of \cite{zhanjon13} which takes its roots in work of \cite{greensmith2004variance}. The main idea behind SVRG is to build an auxiliary sequence $\tilde x$ at which the full gradient is calculated and used as a reference in building a gradient estimate: $\nabla f_i(x)-\nabla f_i(\tilde x) + \nabla f(\tilde x)$. Again this estimate is unbiased under the uniform choice of $i$. While using this gradient estimate to optimize sum-decomposable functions, the variance will be small when $x$ and $\tilde x$ are close to the optimum as $\nabla f(\tilde x)$ is \emph{small} and $\lVert \nabla f_i(x)-\nabla f_i(\tilde x)\rVert$ is of the order $\Vert x-\tilde x\Vert_2$. We also expect a similar behavior in the case of Monte Carlo sampling and we thus use this gradient estimate in Algorithm \ref{alg:svrg}. Observe that crucially---unlike SAGA-based algorithms---this method does not require an estimate of all of the individual $f_i$, so the memory cost of this algorithm scales in the worst case as $\mathcal{O}(d)$. In Algorithm \ref{alg:svrg} we use the unbiased gradient estimate
\begin{align}   \vspace{-5pt}
\label{eq:svrggradientupdate}
g_k = \tilde{g} + \frac{N}{n} \sum_{i \in S} \left[\nabla f_i (x_k) - \nabla f_i(\tilde{x})\right],
\vspace{-5pt}
\end{align}
which uses a mini-batch of size $n$.
\label{sec:svrg}
\setlength{\textfloatsep}{10pt}
\begin{algorithm}[t]
\caption{SVRG Langevin MCMC \label{alg:svrg}}
\begin{algorithmic}
   \STATE {\bfseries Input:} Gradient oracles $\{\nabla f_i(\cdot)\}_{i=0}^N$, step size $\delta$, epoch length $\tau$, batch size $n$, initial point $x_0 \in \mathbb{R}^d$.
   \STATE Initialize $\tilde{x} \leftarrow x_0$, $\tilde{g} \leftarrow \sum_{i=1}^N \nabla f_i(x_0)$
   \FOR {$k=1,\ldots,T$}
   \IF {$k$ mod $\tau = 0$}
   \STATE \textbf{Option I:} Sample $\ell \sim unif(0,1,\ldots,\tau-1)$ and Update $\tilde{x} \leftarrow x_{k-\ell}$
   \STATE Update $x_k \leftarrow \tilde{x}$
   	\STATE \textbf{Option II:} Update $\tilde{x} \leftarrow x_{k}$
	\STATE  $\tilde{g} \leftarrow \sum_{i=1}^N \nabla f_i (x_{k})$
   \ENDIF
   \STATE Draw $S \subset \{0,\ldots,N\}: \lvert S \rvert = n$ uniformly with replacement
   \STATE Sample $\xi_k \sim N(0, I_{d\times d})$
   \STATE Update $g_k$ using \eqref{eq:svrggradientupdate}
   \STATE Update $x_{k+1}\leftarrow x_k - \delta g + \sqrt{2\delta}\xi_k$.
   \ENDFOR
   \STATE{\bfseries Output:} Iterates $\{x_k\}_{k=1}^T$.
\end{algorithmic}
\end{algorithm}
The downside of this algorithm compared to SAGA however is that every few steps (an epoch) the full gradient, $\nabla f(\tilde{x})$, needs to be calculated at $\tilde{x}$. This results in the query complexity of each epoch being $\mathcal{O}(N)$. Also SVRG has an extra parameter that needs to be set---its hyperparameters are the epoch length ($\tau$), the step size ($\delta$) and the batch size ($n$), as opposed to just the step size and batch size for Algorithm \ref{alg:saga} which makes it harder to tune. It also turns out that in practice, SVRG seems to be consistently outperformed by SAGA and control-variate techniques for sampling which is observed both in previous work and in our experiments. %We also present a convergence result of this method in the next section. %\textcolor{red}{add description of the two variants. should we ?}

\subsection{Control Variates with Underdamped Langevin MC}
%!TEX root = main.tex
Another approach is to use control variates \citep{ripley2009stochastic} to reduce the variance of stochastic gradients. This technique has also been previously explored both theoretically and experimentally by \cite{baker2017control} and \cite{nagapetyan2017true}. Similarly to SAGA and SVRG the idea is to build an unbiased estimate of the gradient $g(x)$ at a point $x$:
\begin{align*}
g(x) = \nabla f(\hat{x}) + \sum_{i\in S} \left[\nabla f_i(x)-\nabla f_i(\hat{x})\right],
\vspace{-5pt}
\end{align*}
where the set $S$ is the mini-batch and $\hat{x}$ is a fixed point that is called the \emph{centering value}. Observe that taking an expectation over the choice of the set $S$ yields $\nabla f(x)$. A good centering value $\hat{x}$ would ensure that this estimate also has low variance; a natural choice in this regard is the \emph{global minima} of $f$, $x^*$. A motivating example is the case of a Gaussian random variable where the mean of the distribution and $x^*$ coincide.

A conclusion of previous work that applies control variate techniques to stochastic gradient Langevin MCMC is the following---the variance of the gradient estimates can be lowered to be of the order of the discretization error. Motivated by this, we apply these techniques to \emph{underdamped} Langevin MCMC where the underlying continuous time diffusion process is given by the following second-order SDE:
\vspace{-5pt}
\begin{align*}
\numberthis \label{eq:langevindiffusionmaintext}dv_t &= -\gamma v_t dt - u \nabla f(x_t) dt + \sqrt{2}dB_t, \\
dx_t & = v_t dt,
\vspace{-5pt}
\end{align*}
where $(x_t,v_t) \in \mathbb{R}^d$, $B_t$ represents the standard Brownian motion and $\gamma$ and $u$ are constants. At a high level the advantage of using a second-order MCMC method like underdamped Langevin MCMC \cite{cheng2017underdamped}, or related methods like Hamiltonian Monte Carlo \citep[see, e.g, ][]{neal2011mcmc,rhmc}, is that the discretization error is lower compared to overdamped Langevin MCMC.
However when stochastic gradients are used \citep[see][for implementation]{SGHMC,completesample}, this advantage can be lost as the variance of the gradient estimates dominates the total error. We thus apply control variate techniques to this second-order method. This reduces the variance of the gradient estimates to be of the order of the discretization error and enables us to recover faster rates of convergence.
\begin{algorithm}[t]
\caption{CV Underdamped Langevin MCMC \label{alg:cvulmcmc}}
\begin{algorithmic}
   \STATE {\bfseries Input:} Gradient oracles $\{\nabla f_i(\cdot)\}_{i=0}^N$, step size $\delta$, smoothness $M$, batch size $n$.
   \STATE Set $x^* \in \argmin_{x \in \mathbb{R}^d} f(x)$.
   \STATE Set $(x_0,v_0) \leftarrow (x^*,0)$
   \FOR {$k=1,\ldots,T$}
   \STATE Draw a set $S \subset \{0,\ldots N\}$ of size $n$ u.a.r.
   \STATE Update $\nabla \tilde{f}(x_k)$ using \eqref{eq:cvgradupdate}
 	\STATE Sample $(x_{k+1},v_{k+1}) \sim  Z^{k+1}(x_k,v_k)$  defined in \eqref{eq:defofZnormal}
   \ENDFOR
   \STATE{\bfseries Output:} Iterates $\{x_k \}_{k=1}^{T}$.
\end{algorithmic}
\end{algorithm}
The discretization of SDE \eqref{eq:langevindiffusionmaintext} (which we can simulate efficiently) is
\vspace{-5pt}
\begin{align*}
\numberthis \label{eq:langevindiscremaintext}d\tilde{v}_t &= -\gamma \tilde{v}_t dt - u \nabla \tilde{f}(x_k) dt + \sqrt{2}dB_t, \\
d\tilde{x}_t & = \tilde{v}_t dt,
\vspace{-5pt}
\end{align*}
with initial conditions $x_k,v_k$ (the previous iterate of the Markov Chain) and $\nabla \tilde{f}(x_k)$ is the estimate of the gradient at $x_k$, defined in \eqref{eq:cvgradupdate}. We integrate \eqref{eq:langevindiscremaintext} for time $\delta$ (the step size) to get our next iterate of the chain---$x_{k+1},v_{k+1}$ for some $k\in \{1,\ldots,T \}$. This MCMC procedure was introduced and analyzed by \cite{cheng2017underdamped} where they obtain that given access to full gradient oracles the chain converges in $T = \tilde{\mathcal{O}}(\sqrt{d}/\epsilon)$ steps  (without Assumption \ref{ass:HessianLipschitz}) as opposed to standard Langevin diffusion which takes $T = \tilde{\mathcal{O}}(d/\epsilon)$ steps (with Assumption \ref{ass:HessianLipschitz}). With noisy gradients (variance $\sigma^2 d$), however, the mixing time of underdamped Langevin MCMC again degrades to $\tilde{\mathcal{O}}(\sigma^2 d/\epsilon^2)$.

In Algorithm \ref{alg:cvulmcmc} we use control variates to reduce variance and are able to provably recover the fast mixing time guarantee ($T = \tilde{\mathcal{O}}(\sqrt{d}/\epsilon)$) in Theorem \ref{thm:controlvariatethm}. Algorithm \ref{alg:cvulmcmc} requires a pre-processing step of calculating the (approximate) minimum of $f$ as opposed to Algorithm \ref{alg:saga},\ref{alg:svrg}; however since $f$ is strongly convex this pre-processing cost (using say SAGA for optimizing $f$ with stochastic gradients) is small compared to the computational cost of the other steps.
\label{sec:controlvariate}

In Algorithm \ref{alg:cvulmcmc} the updates of the gradients are dictated by,
\begin{align} \label{eq:cvgradupdate}
\nabla \tilde{f}(x_k) = \nabla f(x^*) + \frac{N}{n} \sum_{i \in S}\left [\nabla f_i(x_k) - \nabla f_i(x^*)\right].
\end{align}
The random vector that we draw, $Z^{k}(x_k,v_k) \in \mathbb{R}^{2d}$, conditioned on $x_k,v_k$, is a Gaussian vector with conditional mean and variance that can be \emph{explicitly} calculated in closed form expression in terms of the algorithm parameters $\delta$ and $M$. Its expression is presented in Appendix \ref{app:controlvariates}. Note that $Z^{k}$ is a Gaussian vector and can be sampled in $\mathcal{O}(d)$ time.

\section{Convergence results}
%!TEX root = main.tex
\label{sec:convergenceresults}In this section we provide convergence results of the algorithms presented above, which improve upon the convergence guarantees for SGLD. \cite{dalalyan2017user} show that for SGLD run for $T$ iterations:
\begin{align} \label{eq:wassdala}
W_2(p^{(T)},p^*)  \le \exp \left( -\delta m T \right) W_2(p^{(0)},p^*)
+ \frac{\delta Ld}{2m}+ \frac{11 \delta M^{3/2}\sqrt{d}}{5 m}+\frac{\sigma \sqrt{\delta d} }{2\sqrt{m}},
\end{align}
under assumptions \ref{ass:smoothness}-\ref{ass:HessianLipschitz} with access to stochastic gradients with bounded variance -- $\sigma^2 d$. The term involving the variance -- $\sigma \sqrt{\delta d} /{2\sqrt{m}}$ dominates the others in many interesting regimes. For sum-decomposable functions that we are studying in this paper this is also the case as the variance of the gradient estimate usually scales linearly with $N^2$. Therefore the performance of SGLD sees a deterioration when compared to the convergence guarantees of Langevin Diffusion where $\sigma=0$. To prove our convergence results we follow the general framework established by \cite{dalalyan2017user}, with the noteworthy difference of working with more sophisticated Lyapunov functions (for Theorems \ref{thm:mainsaga} and \ref{thm:svrgtheorem}) inspired by proof techniques in optimization theory. This contributes to strengthening the connection between optimization and sampling methods raised in previous work and may potentially be applied to other sampling algorithms (we elaborate on these connections in more detail in Appendix \ref{app:svrgsaga}). This comprehensive proof technique also allows us to sharpen the convergence guarantees obtained by \cite{dubey2016variance} on variance reduction methods like SAGA and SVRG  by allowing us to present bounds in $W_2$ and to drop the assumption on requiring uniformly bounded gradients. We now present convergence guarantees for Algorithm \ref{alg:saga}.
%The main difference with work of Reddi is: a) our set of assumption is realistic, b) we combine both framework instead of simply plugging one into the other.  We state the following convergence result for SAGA-MCMC
\begin{theorem} \label{thm:mainsaga}
Let assumptions \ref{ass:decom}-\ref{ass:HessianLipschitz} hold. Let $p^{(T)}$ be the distribution of the iterate of Algorithm \ref{alg:saga} after $T$ steps.  If we set the step size to be $\delta<\frac{n}{8M N}$
%$\delta<\min\{ \frac{n}{8M N}, \frac{2n}{3Nm}, \frac{n}{N M} \}$
and the batch size $n \geq 9$ then we have the guarantee:

\begin{align} \label{eq:wassersteinsaga}
W_2(p^{(T)},p^\ast)\leq 5	\exp\left(-\frac{m\delta}{4}T\right)W_2(p^{(0)},p^*)  + \frac{2\delta      Ld}{m}+\frac{2\delta M^{3/2} \sqrt{d}}{m} +\frac{24\delta M \sqrt{ d N} }{\sqrt{m} n}.
\end{align}
\end{theorem}
%\begin{theorem} \label{thm:mainsaga}
%Let assumptions \ref{ass:decom}-\ref{ass:HessianLipschitz} hold. Let $p^{(T)}$ be the distribution of the iterate of Algorithm \ref{alg:saga} after $T$ steps.  If we set the step size to be $\delta<\min\{ \frac{1}{8M} \}$ and the batch size $n \geq 9$ then we have the guarantee
%\begin{align} \label{eq:wassersteinsaga}
%W_2(p^{(T)},p^\ast)\leq 	\exp\left(-\rho T\right)\left[\frac{5\delta M\sqrt{N}}{n}+1\right]W_2(p^{(0)},p^*)\nonumber\\
%+\frac{\sqrt{2m\delta}}{\sqrt{\rho}}\left[ \frac{12\delta M \sqrt{ d N} }{\sqrt{m} n} \left(1+\frac{\sqrt{\delta M N }}{\sqrt{n}}\right) + \frac{\delta      Ld}{m}+\frac{\delta M^{3/2} \sqrt{d}}{m}\right].
%\end{align}
%\end{theorem}
For the sake of clarity, only results for small  step-size $\delta$ are presented however, it is worth noting that convergence guarantees hold for any $\delta\leq\frac{1}{8M}$ (see details in Appendix \ref{app:sagaproof}).
 If we consider the regime where $\sigma,M,L$ and $m$ all scale linearly with the number of samples $N$, then for SGLD the dominating term is $\mathcal{O}(\sigma\sqrt{\delta d/m})$. If the target accuracy is $\epsilon$, SGLD would require the step size to scale as $\mathcal{O}(\epsilon^2/d)$ while for SAGA a step size of $\delta = \mathcal{O}(\epsilon/d)$ is sufficient. The mixing time $T$ for both methods is roughly proportional to the inverse step-size; thus SAGA provably takes fewer iterations while having almost the same computational complexity per step as SGLD. Similar to the optimization setting, theoretically SAGA Langevin diffusion recovers the \emph{fast rate} of Langevin diffusion while just using cheap gradient updates. Next we present our guarantees for Algorithm \ref{alg:svrg}.
%We can make the following remark:
%\begin{itemize}
%\item Obtain the same convergence rate as Langevin diffusion (dont pay the price of the variance).
%\item easy algorithm (only step size to fix)
%
%\end{itemize}
\begin{theorem}\label{thm:svrgtheorem}
Let assumptions \ref{ass:decom}-\ref{ass:HessianLipschitz} hold. Let $p^{(T)}$ be the distribution of the iterate of Algorithm \ref{alg:svrg} after $T$ steps.

If we set $\delta<\frac{1}{8M}$, $n \ge 2$, $\tau \geq \frac{8}{m\delta}$ and run $\mathsf{Option\ I}$ then for all $T$ mod $\tau = 0$ we have
\begin{align} \label{eq:wassersteinsvrg}
W_2(p^{(T)},p^*)  \le \exp\left( -\frac{\delta m T }{56}\right) \frac{\sqrt{M}}{\sqrt{m}} W_2(p^{(0)},p^*)+ \frac{2\delta Ld}{m}+ \frac{2\delta M^{3/2}\sqrt{ d}}{m}+ \frac{64M^{3/2}\sqrt{\delta d}}{m\sqrt{n}}.
\end{align}
If we set $\delta\!<\!\frac{\sqrt{n}}{4\tau M}$ and run $\mathsf{Option\ II}$ for $T$ iterations then,
\begin{align} \label{eq:option2statement}
W_2(p^{(T)},p^*) \leq \exp\left(-\frac{\delta m T }{4}\right)  W_2(p^{(0)},p^*)
+\frac{\sqrt{2}\delta Ld}{m}+\frac{5\delta M^{3/2}\sqrt{d}}{m}+\frac{9\delta M\tau \sqrt{d}}{\sqrt{m n}}.
\end{align}
\end{theorem}
For Option I, if we study the same regime as before where $M,m$ and $L$ are scaling linearly with $N$ we find that the discretization error is dominated by the term which is of order $\mathcal{O}(\sqrt{\delta N d/n})$. To achieve target accuracy of $\epsilon$ we would need $\delta = \mathcal{O}(\epsilon^2 n/Nd)$. This is less impressive than the guarantees of SAGA and essentially we only gain a constant factor as compared to the guarantees for SGLD.  This behavior may be explained as follows: at each epoch, a constant decrease of the objective is needed in the classical proof of SVRG when applied to optimization. When the step-size is small, the epoch length is required to be large that washes away the advantages of variance reduction.

For Option II, similar convergence guarantees as SAGA are obtained, but worse by a factor of $\sqrt{n}$. In contrast to SAGA, this result holds only for small step-size, with the constants in Eq.~\eqref{eq:option2statement} blowing up exponentially quickly for larger step sizes (for more details see proof in Appendix \ref{app:svrgappendix}).
We also find that experimentally SAGA routinely outperforms SVRG both in terms of run-time and iteration complexity to achieve a desired target accuracy. However, it is not clear whether it is an artifact of our proof techniques that we could not recover matching bounds as SAGA or if SVRG is less suited to work with sampling methods. We now state our results for the convergence guarantees of Algorithm \ref{alg:cvulmcmc}.
\begin{theorem} \label{thm:controlvariatethm} Let assumptions \ref{ass:decom}-\ref{ass:strongconvexity} hold. Let $p^{(T)}$ be the distribution of the iterate of Algorithm \ref{alg:cvulmcmc} after $T$ steps starting with the initial distribution $p^{(0)}(x,v) = 1_{x=x^*}\cdot 1_{v=0}$. If we set the step size to be $\delta < 1/M$ and run Algorithm \ref{alg:cvulmcmc} then we have the guarantee that
\begin{align}
W_2(p^{(T)},p^*) &\le 4\exp\left(-\frac{m\delta T}{2} \right)W_2(p^{(0)},p^*) +\frac{ 164 \delta M^2  \sqrt{d}}{m^{3/2}}+  \frac{83M \sqrt{d}}{m^{3/2}\sqrt{ n}}.\label{eq:wassersteincontractionalgcv} 
\end{align}
%\begin{multline*}
%W_2(p^{(T)},p^*) \le 4\exp\left(-\frac{\delta T}{2\kappa} \right)W_2(p^{(0)},p^*) \\ + 2\kappa \delta \sqrt{\frac{1664d}{5m}}+ 4\kappa \sqrt{\frac{1280d}{m n}}.\label{eq:wassersteincontractionalgcv} \numberthis
%\end{multline*}
\end{theorem}
%\begin{remark} Note that no attempt has been made to optimize the constants in the theorems stated above.
%\end{remark}
We initialize the chain in Algorithm \ref{alg:cvulmcmc} with $x^*$, the global minimizer of $f$ as we already need to calculate it to build the gradient estimate.
Observe that Theorem \ref{thm:controlvariatethm} does not guarantee the error drops to $0$ when $\delta \to 0$ but is proportional to the standard deviation of our gradient estimate. This is in contrast to SAGA and SVRG based algorithms where a more involved gradient estimate is used. The advantage however of using this second order method is that we get to a desired error level $\epsilon$ at a faster rate as the step size can be chosen proportional to $\epsilon/\sqrt{d}$, which is $\sqrt{d}$ better than the corresponding results of Theorem \ref{thm:mainsaga} and \ref{thm:svrgtheorem} and without Assumption \ref{ass:HessianLipschitz} (Hessian Lipschitzness).
%\note	{say clearly that the algorithm is not converging to 0 but only to a neighboorhood of the solution of radius proportional of the variance of the estimate of the gradient, and this method needs to take a  }

Note that by Lemma \ref{lem:initialdistancebound} we have the guarantee that $W_2(p^{(0)},p^*) \le 2d/m$; this motivates the choice of $\delta =\mathcal{O}\left(\frac{m \epsilon\sqrt{m}}{M^2\sqrt{d}}\right)$ and, $n =\mathcal{O}\left( \frac{ M^2 d}{m^3\epsilon^2}\right)$ with $T =\tilde{\mathcal{O}}\left( 1/(m\delta))\right)$. It is easy to check that under this choice of $\delta,n$ and $T$, Theorem \ref{thm:controlvariatethm} guarantees that $W_2(p^{(T)},p^*) \le \epsilon$. We note that no attempt has been made to optimize the constants. To interpret these results more carefully let us think of the case when $M,m$ both scale linearly with the number of samples $N$. Here the number of steps $T = \tilde{\mathcal{O}}(\sqrt{d/(\epsilon^2 N)})$ and the batch size is $n = \mathcal{O}(d/(N\epsilon^2))$. If we compare it to previous results on control variate variance reduction techniques applied to \emph{overdamped} Langevin MCMC by \cite{baker2017control}, the corresponding rates are $T=\tilde{\mathcal{O}}(d/(\epsilon^2N))$ and $n = \mathcal{O}(d/(N\epsilon^2))$, essentially it is possible to get a quadratic improvement by using a second order method even in the presence of noisy gradients. Note however that these methods are not viable when the target accuracy $\epsilon$ is small as the batch size $n$ needs to grow as $\mathcal{O}(1/\epsilon^2)$.

%\textcolor{red}{add comparison to sgld, compare dependence on d and epsilon for large N.}

%\subsection*{Comparison of Methods}
%!TEX root = main.tex
\label{sec:comparisonsection}
\begin{table}[t]
\caption{Mixing time and computational complexity comparison of Langevin sampling algorithms. All the entries in the table are in Big-O notation which hides constants and poly-logarithmic factors. Note that the guarantees presented for ULD, SGULD, CV-LD and CV-ULD are without the Hessian Lipschitz assumption, \ref{ass:HessianLipschitz}.}
\label{tab:comparisontable}
\vskip 0.15in
\begin{center}
\begin{footnotesize}
\begin{sc}
\begin{tabular}{lcccr}
\toprule
Algorithm & Mixing Time & Computation \\
\midrule
LD &  $\kappa^2\sqrt{d}/(\sqrt{N}\epsilon)$ & $\kappa^2\sqrt{dN}/\epsilon$\\
ULD & $\kappa^{\frac{5}{2}}\sqrt{d}/(\sqrt{N}\epsilon)$ & $\kappa^{\frac{5}{2}}\sqrt{d N}/\epsilon$\\
\midrule[0.2pt]
SGLD    & $\kappa^2d/(n\epsilon^2)$& $\kappa^2 d/\epsilon^2$ \\
SGULD & $\kappa^2d/(n\epsilon^2)$& $\kappa^2 d/\epsilon^2$  \\
\midrule[0.2pt]
SAGA-LD    & $\kappa^{\frac{3}{2}}\sqrt{d}/(n\epsilon)$& $N\!+\!\kappa^{\frac{3}{2}}\sqrt{d}/\epsilon$\\
SVRG-LD (I)    &$\kappa^3d/(n\epsilon^2)$& $N\!+\! \kappa^3d/\epsilon^2$   \\
SVRG-LD (II)  & $\kappa^{\frac{11}{6}}\sqrt{d}/(N^{\frac{2}{3}}\epsilon)$& $N\!+\!\kappa^{\frac{5}{3}}N^{\frac{1}{6}}\sqrt{d}/\epsilon$ \\
\midrule[0.2pt]
CV-LD    & $\kappa^3d/(N\epsilon^2)$& $N\!+\!\kappa^6 d^2/(N^2\epsilon^4)$\\
CV-ULD     & $\kappa^{\frac{5}{2}}\sqrt{d}/(\sqrt{N}\epsilon)$ &  $N\!+\!\kappa^{\frac{11}{2}}d^{\frac{3}{2}}/(N^{\frac{3}{2}}\epsilon^3)$\\
\bottomrule
\end{tabular}
\end{sc}
\end{footnotesize}
\end{center}
\vskip -0.1in
\end{table}
\paragraph{Comparison of Methods.} Here we compare the theoretical guarantees of Langevin MCMC \citep[LD, ][]{durmus2016sampling}, Underdamped Langevin MCMC \citep[ULD, ][]{cheng2017underdamped}, SGLD \citep{dalalyan2017user}, stochastic gradient underdamped Langevin diffusion  \citep[SGULD,][]{cheng2017underdamped}, SAGA-LD (Algorithm \ref{alg:saga}), SVRG-LD (Algorithm \ref{alg:svrg} with Option I and II), Control Variate Langevin diffusion \citep[CV-LD,][]{baker2017control} and Control Variate underdamped Langevin diffusion (CV-ULD, Algorithm \ref{alg:cvulmcmc}). We always consider the scenario where $M,m$ and $L$ are scaling linearly with $N$ and where $N \gg d$ (\emph{tall-data} regime). We note that the memory cost of all these algorithms except SAGA-LD is $\mathcal{O}(nd)$; for SAGA-LD the worst-case memory cost scales as $\mathcal{O}(Nd)$. Next we compare the mixing time ($T$), that is, the number of steps needed to provably have error less than $\epsilon$ measured in $W_2$ and the computational complexity, which is the mixing time $T$ times the query complexity per iteration. In the comparison below we focus on the dependence of the mixing time and computational complexity on the dimension $d$, number of samples $N$, condition number $\kappa$, and the target accuracy $\epsilon$. The mini-batch size has no effect on the computational complexity of SGLD, SGULD and SAGA-LD; while for SVRG-LD, CV-LD and CV-ULD the mini-batch size is chosen to optimize the upper bound.
%\begin{table}[h]
%\caption{Mixing time and computational complexity comparison of Langevin sampling algorithms. All the entries in the table are in Big-O notation which hides constants and poly-logarithmic factors.}
%\label{tab:comparisontable}
%\vskip 0.15in
%\begin{center}
%\begin{small}
%\begin{sc}
%\begin{tabular}{lcccr}
%\toprule
%Algorithm & Mixing Time & Computation \\
%\midrule
%LD &  $\kappa^2\sqrt{d}/(\sqrt{N}\epsilon)$ & $\kappa^2\sqrt{dN}/\epsilon$\\
%ULD (w.o. \ref{ass:HessianLipschitz}) & $\kappa^{5/2}\sqrt{d}/(\sqrt{N}\epsilon)$ & $\kappa^{5/2}\sqrt{d N}/\epsilon$\\
%SGLD    & $\kappa^2d/(nN\epsilon^2)$& $\kappa^2 d/(N\epsilon^2)$ \\
%SGULD (w.o. \ref{ass:HessianLipschitz})& $\kappa^2d/(nN\epsilon^2)$& $\kappa^2 d/(N\epsilon^2)$  \\
%SAGA-LD    & $\kappa^{3/2}\sqrt{d}/(n\epsilon)$& $\kappa^{3/2}\sqrt{d}/\epsilon$\\
%SVRG-LD (Option I)    &$\kappa^3d/(n\epsilon^2)$& $N+ \kappa^3d/\epsilon^2$   \\
%SVRG-LD (Option II)  & $\kappa^{11/6}\sqrt{d}/(N^{2/3}\epsilon)$& $\kappa^{5/3}N^{1/6}\sqrt{d}/\epsilon$ \\
%CV-LD  (w.o. \ref{ass:HessianLipschitz})   & $\kappa^3d/(N\epsilon^2)$& $\kappa^6 d^2/(N^2\epsilon^4)$\\
%CV-ULD (w.o. \ref{ass:HessianLipschitz})     & $\kappa^{5/2}\sqrt{d}/(\sqrt{N}\epsilon)$ &  $\kappa^{11/2}d^{3/2}/(N^{3/2}\epsilon^3)$\\
%\bottomrule
%\end{tabular}
%\end{sc}
%\end{small}
%\end{center}
%\vskip -0.1in
%\end{table}

\begin{figure}[t]
\centering
 \def\svgwidth{3.1in}
\input{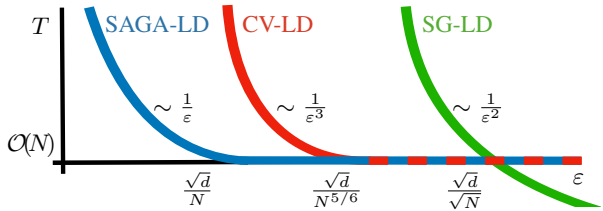}
\caption[Different regimes]{Different Regimes: The $x$-axis represents the target accuracy $\epsilon$ and the $y$-axis represents the predicted run-time $T$ (number of queries to the gradient oracle) of different algorithms.}
\vspace{-5pt}
\label{fig:regime}
\end{figure}

As illustrated in Fig.~\ref{fig:regime} we see a qualitative difference in behavior of variance reduced algorithms compared to methods like SGLD. In applications like calculating higher order statistics or computing confidence intervals to quantify uncertainty it is imperative to calculate the posterior with very high accuracy. In this regime when the target accuracy $\epsilon < \mathcal{O}(\sqrt{d/N})$, the computational complexity of SGLD starts to grow larger than $\mathcal{O}(N)$ at rate $\mathcal{O}(\sqrt{d}/\epsilon^2)$ whereas the computational cost of variance reduced methods is lower. For SAGA-LD the computational cost is $\mathcal{O}(N)$ up until when $\epsilon = \mathcal{O}(\sqrt{d}/N)$ after which it grows at a rate $\mathcal{O}(\sqrt{d}/\epsilon)$. CV-ULD also has a computational cost of $\mathcal{O}(N)$ up until the point where $\epsilon = \mathcal{O}(\sqrt{d}/N^{5/6})$ after which it starts to grow as $\mathcal{O}(d^{3/2}/(N^{3/2}\epsilon^3))$. When $\mathcal{O}(\sqrt{d}/N^{5/6})\le \epsilon<\mathcal{O}(\sqrt{d}/\sqrt{N})$ our bounds predict both SAGA-LD and CV-ULD to have comparative performance ($\mathcal{O}(N)$) and in some scenarios one might outperform the other. For higher accuracy our results predict SAGA-LD performs better than CV-ULD. Note that Option II of SVRG performs also well in this regime of small $\epsilon$ but not as well as SAGA-LD or CV-ULD.

At the other end of the spectrum for most classical statistical problems accuracy of $\epsilon = \mathcal{O}(\sqrt{d/N})$ is sufficient and less than a single pass over the data is enough. In this regime when $\epsilon > \mathcal{O}(\sqrt{d/N})$ and we are looking to find a crude solution quickly, our bounds predict that SGLD is the fastest method. Other variance reduction methods need at least a single pass over the data to initialize.

Our sharp theoretical bounds allow us to classify and accurately identify regimes where the different variance reduction algorithms are efficient; bridging the gap between experimentally observed phenomenon and theoretical guarantees of previous works. Also noteworthy is that here we compare the algorithms only in the tall-data regime which grossly simplifies our results in Sec.~\ref{sec:convergenceresults}, many other interesting regimes could be considered, for example the \emph{fat-data} regime where $d \approx N$, but we omit this discussion here.

\section{Experiments}
\label{sec:exp}
\label{sec:experiments}

\begin{figure*}[h!]
\centering
\includegraphics[width=0.4\textheight]{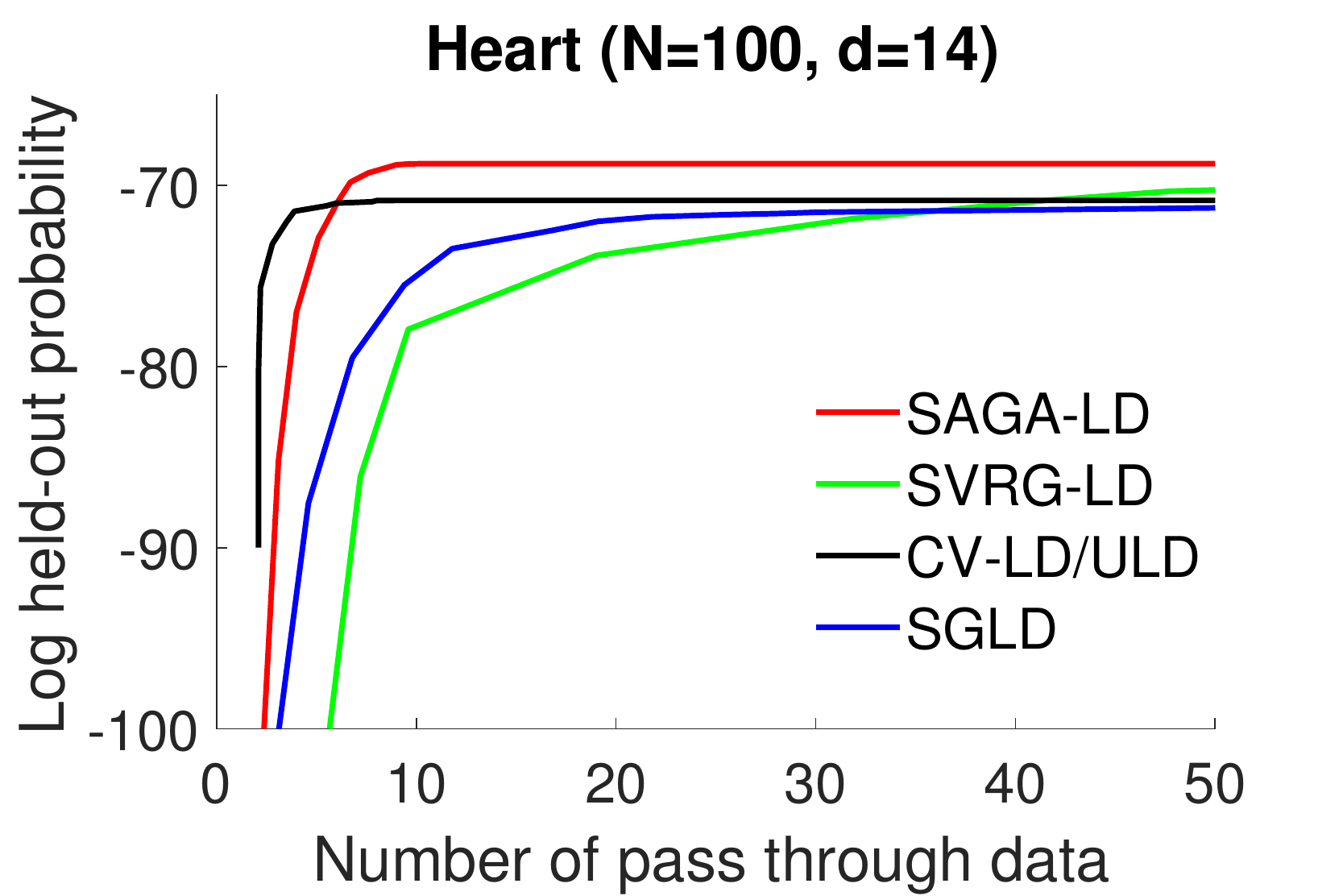} \hspace{-8pt}
\includegraphics[width=0.4\textheight]{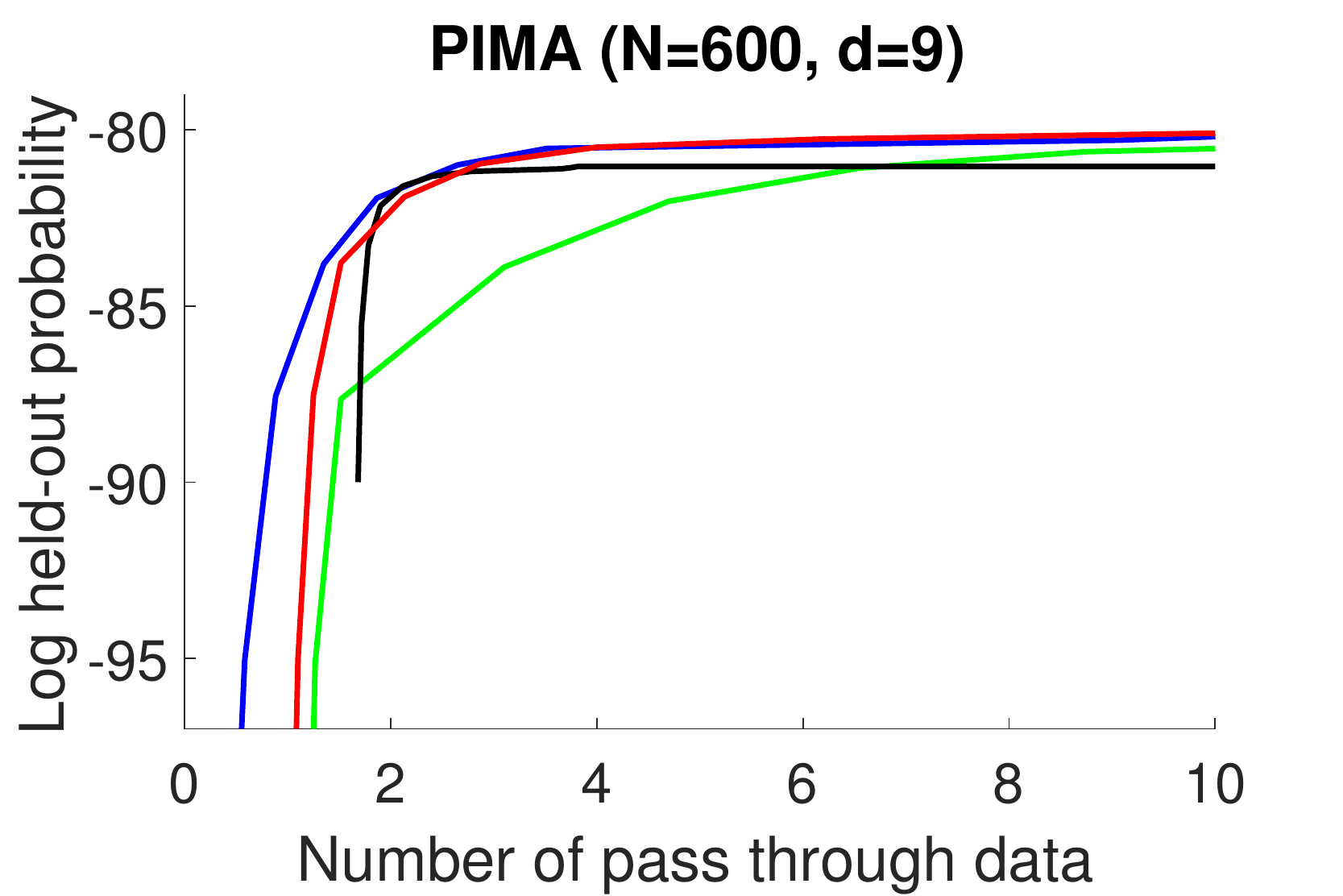} \hspace{-8pt}
\includegraphics[width=0.4\textheight]{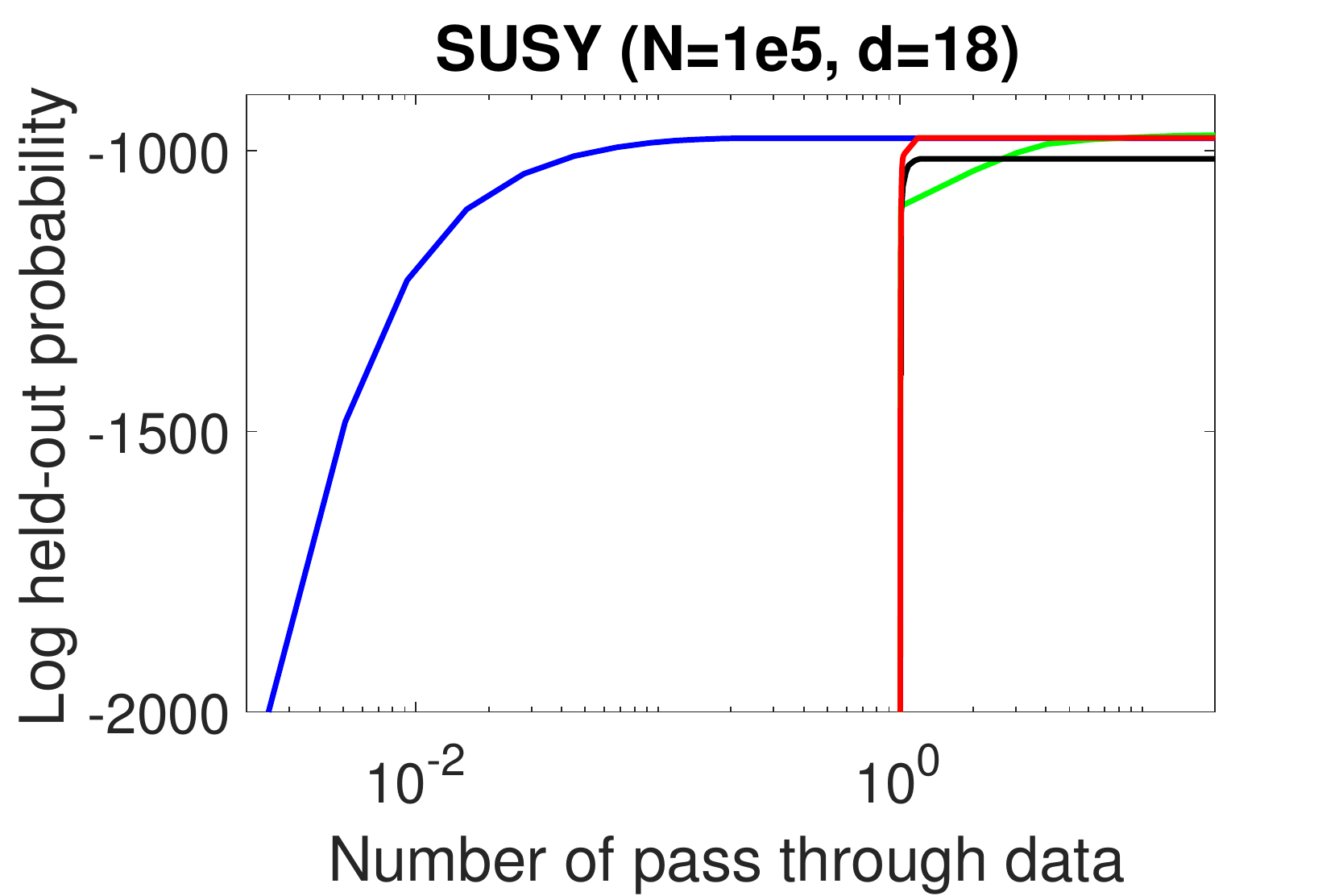}
\caption{
Number of passes through the datasets versus log held-out probability on test datasets.
}
\label{fig:logit}
\end{figure*}

In this section we explore the performance of SG-MCMC with variance reduction via experiments.
We compare SAGA-LD, SVRG-LD (with option II), CV-LD, CV-ULD and use SGLD as the baseline method.

\subsection{Bayesian Logistic Regression}
\label{sec:BayesLogisticRegression}
We demonstrate results from sampling a Bayesian logistic regression model.
We consider an $N \times d$ design matrix ${\bf X}$ comprised of $N$ samples each with $d$ covariates and a binary response variable ${\bf y} \in\{0,1\}^{N}$ \cite{Gelman}.
If we denote the logistic link function by $s(\cdot)$, a Bayesian logistic regression model of the binary response with likelihood $P({\bf y}_i = 1) = s(\beta^T {\bf X}_i)$ is obtained by introducing regression coefficients $\beta\in\mathbb{R}^d$ with a Gaussian prior $\beta\sim\mathcal{N}(0, \alpha I)$, where $\alpha=1$ in the experiments.

We make use of three datasets available at the UCI machine learning repository.
The first two datasets describe the connections between heart disease and diabetes with various patient-specific covariates.
The third dataset captures the generation of supersymmetric particles and its relationship with the kinematic properties of the underlying process.
We use part of the datasets to obtain a mean estimate of the parameters and hold out the rest to test their likelihood under the estimated models.
Sizes of the datasets being used in Bayesian estimation are $100$, $600$, and $1e5$, respectively.

Performance is measured by the log probability of the held-out dataset under the trained model.
We first find the optimal log held-out probability attainable by all the currently methods being tested.
We then target to obtain levels of log held-out probability increasingly closer to the optimal one with each methods.
We record number of passes through data that are required for each method to achieve the desired log held-out probability (averaged over $30$ trials) for comparison in Fig.~\ref{fig:logit}.
We fix the batch size $n=10$ as constant, to explore whether the overall computational cost for SG-MCMC methods can grow sub-linearly (or even be constant) with the overall size of the dataset $N$.
A grid search is performed for the optimal hyperparameters in each algorithm, including an optimal scheduling plan of decreasing stepsizes.
For CV-LD, we first use a stochastic gradient descent with SAGA variance reduction method to find the approximate mode $x^*$.
We then calculate the full data gradient at $x^*$ and initialize the sampling algorithm at $x^*$.

From the experiments, we recover the three regimes displayed in Fig.~\ref{fig:regime} with different data size $N$ and accuracy level with error $\epsilon$.
When $N$ is large, SGLD performs best for big $\epsilon$.
When $N$ is small, CV-LD/ULD is the fastest for relatively big $\epsilon$.
When $N$ and $\epsilon$ are both small so that many passes through data are required, SAGA-LD is the most efficient method.
It is also clear from Fig.~\ref{fig:logit} that although CV-LD/ULD methods initially converges fast, there is a non-decreasing error (with the constant mini-batch size) even after the algorithm converges (corresponding to the last term in Eq.~\eqref{eq:wassersteincontractionalgcv}).
Because CV-LD and CV-ULD both converge fast and have the same non-decreasing error, their performance overlap with each other.
Convergence of SVRG-LD is slower than SAGA-LD, because the control variable for the stochastic gradient is only updated every epoch.
This attribute combined with the need to compute the full gradient periodically makes it less efficient and costlier than SAGA-LD.
We also see that number of passes through the dataset required for SG-MCMC methods (with and without variance reduction) is decreasing with the dataset size $N$. Close observation shows that although the overall computational cost is not constant with growing $N$, it is sublinear.
%We also observe that when $N$ is small and many passes through data are required, SAGA-LD converges faster than SGLD;
%when $N$ is larger, cost for one full gradient calculation dominates the computation cost for large $N$, and SGLD becomes more preferable.
%This fact corroborates with our theoretical finding of the computational cost of SAGA-LD versus SGLD as $N+\kappa^{3/2}\sqrt{d}/\epsilon$ versus $\kappa^{2}d/\epsilon^2$.

%The CV-LD and CV-ULD methods initially converges the fastest.
%But with a constant minibatch size, there is a non-decreasing error even after the algorithm converges (corresponding to the last term in Eq.~\eqref{eq:wassersteincontractionalgcv}).
%Because CV-LD and CV-ULD converges fast and have the same non-decreasing error, their performance overlap with each other.
%Overall, we recover that SGLD, SAGA-LD, and CV-LD are most efficient in different regimes, depending on $N$ and $\epsilon$ (See Fig.~\ref{fig:regime}).

\subsection{Breaking CLT: Synthetic Log Normal Data}
Many works using SG-MCMC assume that the data in the mini-batches follow the central limit theorem (CLT) such that the stochastic gradient noise is Gaussian.
But as explained by \cite{remi}, if the dataset follows a long-tailed distribution, size of the mini-batch needed for CLT to take effect may exceed that of the entire dataset.
We study the effects of breaking this CLT assumption on the behaviors of SGLD and its variance reduction variants.

We use synthetic data generated from a log normal distribution: $f_X(x) = \dfrac{1}{x}\cdot\dfrac{1}{\sigma\sqrt{2\pi}} \exp\left(-\dfrac{(\ln x - \mu)^2}{2\sigma^2}\right)$ and sample the parameters $\mu$ and $\sigma$ according to the likelihood $p({\bf x}|\mu,\sigma) = \prod_{i=1}^N f_X(x_i)$.
It is worth noting that this target distribution not only breaks the CLT for a wide range of mini-batch sizes, but also violates assumptions \ref{ass:smoothness}-\ref{ass:HessianLipschitz}.

To see whether each method can perform well when CLT assumption is greatly violated, we still let mini-batch size to be $10$ and grid search for the optimal hyperparameters for each method.
We use mean squared error (MSE) as the convergence criteria and take LD as the baseline method to compare and verify convergence.

From the experimental results, we see that SGLD does not converge to the target distribution.
This is because most of the mini-batches only contain data close to the mode of the log normal distribution.
Information about the tail is hard to capture with stochastic gradient.
It can be seen that SAGA-LD and SVRG-LD are performing well because history information is recorded in the gradient so that data in the tail distribution is accounted for.
As in the previous experiments, CV-LD converges fastest at first, but retains a finite error.
For LD, it converges to the same accuracy as SAGA-LD and SVRG-LD after $10^4$ number of passes through data.
The variance reduction methods which uses long term memory may be especially suited to this scenario, where data in the mini-batches violates the CLT assumption.

It is also worth noting that the computation complexity for this problem is higher than our previous experiments.
Number of passes through the entire dataset is on the order of $10^2\sim10^3$ to reach convergence even for SAGA-LD and SVRG-LD.
It would be interesting to see whether non-uniform subsampling of the dataset \cite{non_uniform} can accelerate the convergence of SG-MCMC even more.

\begin{figure}[h!]
\centering
\includegraphics[width=0.4\textheight]{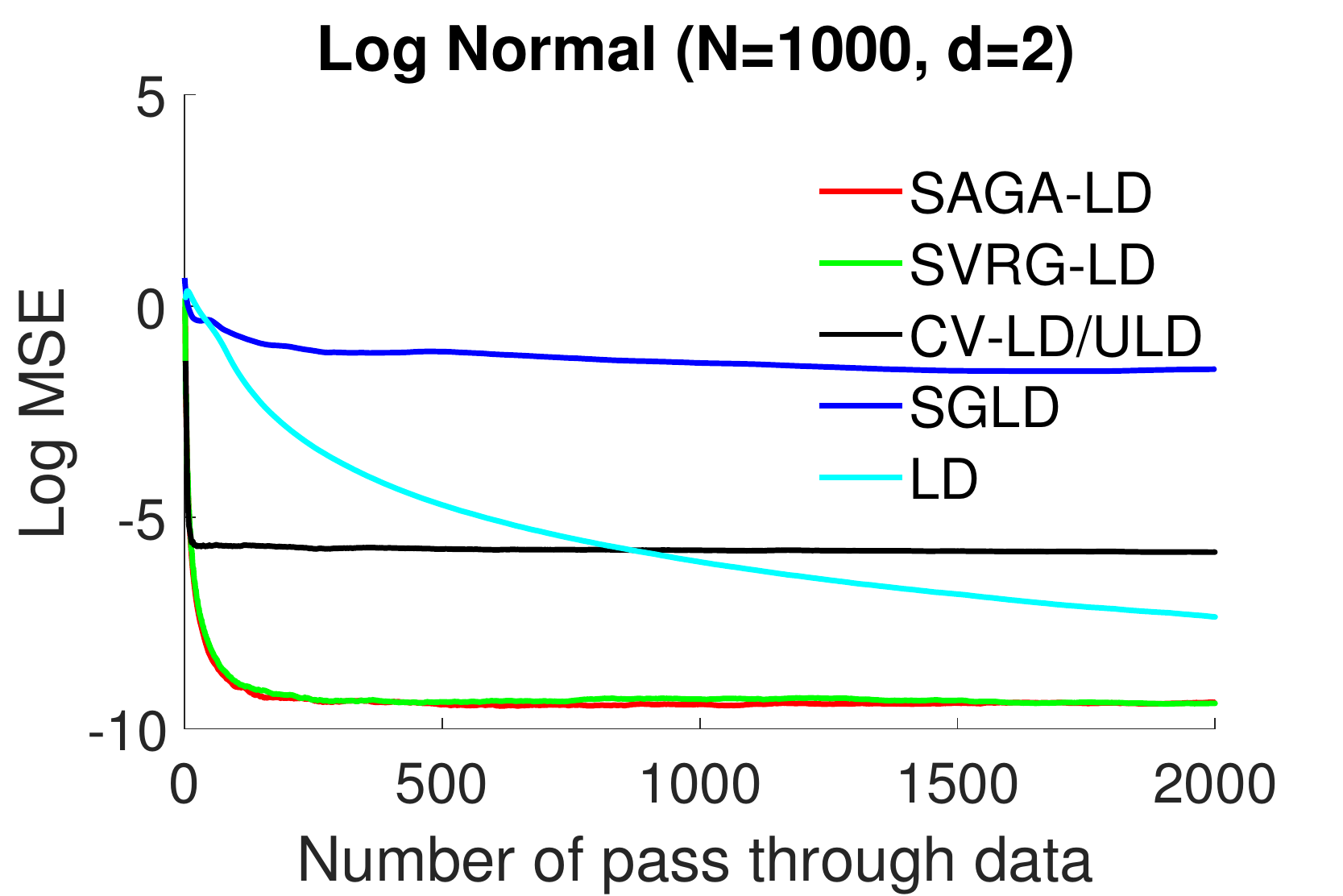}
\vspace{-5pt}
\caption{
Number of passes through the datasets versus log mean square error (MSE).
}
\label{fig:lognormal}
\end{figure}

\section{Conclusions}
In this paper, we derived new theoretical results for variance-reduced stochastic gradient MC.  Our theory allows us to accurately classify two major regimes. When a low-accuracy solution is desired and less than one pass on the data is sufficient, SGLD should be preferred. When high accuracy is needed, variance-reduced methods are much more powerful. There are a number of further directions worth pursuing.  It would be 
of interest to connect sampling with advances in finite-sum optimization. specifically advances in accelerated gradient \citep{LinMaiHar15} or single-pass methods~\citep{LeiJor16}. Finally the development of a theory of lower bounds for sampling 
will be an essential counterpart to this work.
\nocite{*}
\bibliography{ref}

\newpage
\onecolumn
\appendix
\flushleft{\textbf{\Large Appendix}}
\subsubsection*{Organization of the Appendix}
In Appendix \ref{app:wass} we formally define the Wasserstein distance. In Appendix \ref{app:svrgsaga} we introduce the notations required to prove Theorems \ref{thm:mainsaga} and \ref{thm:svrgtheorem}. In Appendix \ref{app:svrgappendix} we prove Theorem \ref{thm:svrgtheorem} and then in Appendix \ref{app:sagaproof} we prove Theorem \ref{thm:mainsaga}. Finally in Appendix \ref{app:controlvariates} we prove Theorem \ref{thm:controlvariatethm}.
\section{Wasserstein Distance} \label{app:wass}
We formally define the Wasserstein distance in this section. Denote by $\mathcal{B}(\mathbb{R}^d)$ the Borel $\sigma$-field of $\mathbb{R}^d$. Given probability measures $\mu$ and $\nu$ on $(\mathbb{R}^d,\mathcal{B}(\mathbb{R}^d))$, we define a \emph{transference plan} $\zeta$ between $\mu$ and $\nu$ as a probability measure on $(\mathbb{R}^d \times \mathbb{R}^d,\mathcal{B}(\mathbb{R}^d\times \mathbb{R}^d))$ such that for all sets $A \in \mathbb{R}^d$, $\zeta(A\times \mathbb{R}^d) = \mu(A)$ and $\zeta( \mathbb{R}^d \times A) = \nu(A)$. We denote $\Gamma(\mu,\nu)$ as the set of all transference plans. A pair of random variables $(X,Y)$ is called a coupling if there exists a $\zeta \in \Gamma(\mu,\nu)$ such that $(X,Y)$ are distributed according to $\zeta$. With some abuse of notation, we will also refer to $\zeta$ as the coupling. 

We define the Wasserstein distance of order two between a pair of probability measures as follows:
\begin{align*}
W_2(\mu,\nu) :=\left(\inf_{\zeta\in\Gamma(\mu,\nu)} \int \lVert x-y\rVert_2^2 d\zeta(x,y) \right)^{1/2}.
\end{align*}
Finally we denote by $\Gamma_{opt}(\mu,\nu)$ the set of transference plans that achieve the infimum in the definition of the Wasserstein distance between $\mu$ and $\nu$ \citep[see, e.g.,][for more properties of $W_2(\cdot,\cdot)$]{villani}.

\section{SVRG and SAGA: Proofs and Discussion}
%!TEX root = main.tex
\label{app:svrgsaga}In this section we will prove Theorem \ref{thm:mainsaga} and Theorem \ref{thm:svrgtheorem} and include details about Algorithms \ref{alg:saga} and \ref{alg:svrg} that were omitted in our discussion in the main paper. Throughout this section we assume that assumptions \ref{ass:decom}-\ref{ass:HessianLipschitz} holds. First we define the continuous time (overdamped) Langevin diffusion process defined by the It\^{o} SDE:
\begin{align}
\label{eq:contilangevin}
dx_t = - \nabla f(x_t) dt + \sqrt{2}dB_t,
\end{align}
here $x_t \in \mathbb{R}^d$, $B_t$ is a standard Brownian motion process and $f(\cdot)$ is a drift added to the process, with the initial condition that $x_0 \sim p_0$. Under fairly mild assumptions, for example if $\exp(-f(x)) \in L^{1}$ (absolutely integrable), then the unique invariant distribution of the process \eqref{eq:contilangevin} is $p^*(x) \propto \exp(-f(x))$. The Euler-Mayurama discretization of this process can be denoted by the It\^{o} SDE:
\begin{align} \label{eq:disclangevin}
d\tilde{x}_t = - \nabla f(\tilde{x}_0) dt + \sqrt{2}dB_t.
\end{align}
Note that in contrast to the process \eqref{eq:contilangevin}, the process \eqref{eq:disclangevin} is driven by the drift evaluated at a fixed initial point $\tilde{x}_0$. In our case we don't have access to the gradient of function $f(\cdot)$, but only to an unbiased estimate of the function, $g_k$ (as defined in Eq.~\eqref{eq:sagagradientupdate} and Eq.~\eqref{eq:svrggradientupdate} for $k \in \{1,\ldots, T\}$). This gives rise to an It\^{o} SDE:
\begin{align}\label{eq:SVRGlangevin}
d\hat{x}_t = -g_k dt + \sqrt{2}dB_t.
\end{align}
Throughout this section we will denote by $\{x_k\}_{k=1}^{T}$ the iterates of Algorithm \ref{alg:saga} or Algorithm \ref{alg:svrg}. Also we will define the distribution of the $k^{th}$ iterate of  Algorithm \ref{alg:saga} or Algorithm \ref{alg:svrg} by $p^{(k)}$. With this notation in place we are now ready to present the proofs of Theorem \ref{thm:svrgtheorem} and Theorem \ref{thm:mainsaga}.
\subsubsection*{Proof Overview and Techniques}
In both the proof of Theorem \ref{thm:mainsaga} and Theorem \ref{thm:svrgtheorem} we draw from and sharpen techniques established in the literature of analyzing Langevin MCMC methods and variance reduction techniques in optimization. In both the proofs we use Lyapunov functions that are standard in the optimization literature for analyzing these methods; we use them to define Wasserstein distances and adapt methods from analysis of sampling algorithms to proceed. 
\subsection{Stochastic Variance Reduced Gradient Langevin Monte Carlo}
%!TEX root = main.tex
\label{app:svrgappendix}In the proof of SVRG for Langevin diffusion, it is common to consider the Lyapunov function to the standard 2-norm. We define a Wasserstein distance with respect to distance -- $\lv x_k - y_k \rv_2^2$.
\begin{proof}[Proof of Theorem \ref{thm:svrgtheorem}]
For any $k \in \{u\tau,\ldots,(u+1)\tau\}$ for some integer $u\in \{0,\ldots,\floor*{T/\tau} \}$, let $y_k$ be a random vector drawn from $p^*$ such that it is optimally coupled to $x_k$, that is, $W_2^2(p^{(k)},p^*) = \mathbb{E}\left[\lv y_k - x_k \rv_2^2\right]$. We also define $\tilde{x}^{u}$ and $\tilde{y}^{u}$ (corresponding to the $u^{th}$ instance of $\tilde{x}$ being updated) analogously, such that $\tilde{y}^{u}\sim p^*$ and $\tilde{x}^{u}$ and $\tilde{y}^{u}$ are optimally coupled. We drop the $u$ in superscript in the proof to simplify notation. We also assume that the random selection of the set of indices to be updated $S$ ($S$ depends on the iteration $k$, but we drop this dependence in the proof to simplify notation) in Algorithm \ref{alg:svrg} is independent of $y_k$. We evolve the random variable $y_k$ under the continuous process described by \eqref{eq:contilangevin},
\begin{align*}
dy_t = -\nabla f(y_t) dt + \sqrt{2}dB_t,
\end{align*}
where the Brownian motion is independent of $(x_k,y_k,S)$. Thus integrating the above SDE we get upto time $\delta$ (the step-size),
\begin{align}\label{eq:ydeltadef}
y_{k+1} = y_k - \int_{0}^{\delta} \nabla f(y_s) ds + \sqrt{2\delta}\xi_k, \qquad \forall \delta>0,
\end{align}
where $\xi_k \sim N(0,I_{d\times d})$. Note that since $y_k \sim p^*$, we have that $y_{k+1} \sim p^*$. Similarly we also have that the next iterate $x_{k+1}$ is given by
\begin{align}\label{eq:xdeltadef}
x_{k+1} = x_{k} -\delta g_k + \sqrt{2\delta}\xi_k,
\end{align}
where $\xi_k$ is the same normally distributed random variable as in \eqref{eq:ydeltadef}. Let us define $\Delta_k := y_k - x_k$ and $\Delta_{k+1} := y_{k+1} - x_{k+1}$.  %We will strive to control the difference $\Delta_{k+1}-\Delta_{k}$; a quantity that characterizes the amount by which the processes evolved under the drift $\nabla f(y_s)$ and $g(x_k)$ diverge in a short time ($\delta$).
Also define
\begin{align*}
V_k : = \int_{k\delta}^{(k+1)\delta} \left(\nabla f(y_s) - \nabla f(y_k) - \sqrt{2}\int_{k\delta}^{s}\nabla^2 f(y_r)dB_r \right)ds.
\end{align*}
Now that we have the notation setup, we will prove the first part of this Theorem. We procede in 5 steps. In Step 1 we will express $\lv \Delta_{k+1} \rv_2^2$ in terms of $\lv \Delta_{k+1} + V_k \rv_2^2$ and $\lv V_k \rv_2^2$, in Step 2 we will control the expected value of $\lv V_k \rv_2^2$. In Step 3 we will express $\Delta_{k+1}$ in terms of $\Delta_k$ and other terms, while in Step 4 we will use the characterization of $\Delta_{k+1}$ in terms of $\Delta_k$ combined with the techniques established by \cite{dubey2016variance} to bound the expected value of $\lv \Delta_{k+1} + V_k \rv_2^2$. Finally in Step 5 we will put this all together and establish our result. First we prove the result for Algorithm \ref{alg:svrg} run with Option I.

\textbf{Step 1:} By Young's inequality we have that $\forall a>0$,
\begin{align} \label{eq:mastersvrgcontrol1}
\lv \Delta_{k+1} \rv_2^2 & = \lv \Delta_{k+1} +V_k - V_k\rv_2^2 \le \left(1+ a\right)\lv \Delta_{k+1} + V_k \rv_2^2 + \left(1+\frac{1}{a} \right) \lv V_k \rv_2^2.
\end{align}
We will choose $a$ at a later stage in the proof to minimize the bound on the right hand side.

\textbf{Step 2:} By Lemma 6 of \cite{dalalyan2017user} we have the bound,
\begin{align}
\mathbb{E}\left[\lv V_k \rv_2^2 \right] \le \left(\frac{\delta^2 L d}{2}+ \frac{\delta^2 M^{3/2}\sqrt{d}}{2} \right)^2 \le \frac{\delta^4}{2}\left(L^2d^2 + M^3 d \right).\label{eq:Vbound}
\end{align}
\textbf{Step 3:} Next we will bound the other term in \eqref{eq:mastersvrgcontrol1}, $\lv \Delta_{k+1} + V_k \rv_2^2$. First we express $\Delta_{k+1}$ in terms of $\Delta_k$,
\begin{align*}
\Delta_{k+1} &= \Delta_{k} + (y_{k+1}- y_k) - (x_{k+1}-x_k) \\
& = \Delta_k + \left(-\int_{k\delta}^{(k+1)\delta}\nabla f(y_s) ds + \sqrt{2\delta}\xi_k \right) - \left(-\int_{k \delta}^{(k+1)\delta} g_k ds + \sqrt{2\delta}\xi_k \right) \\
& = \Delta_k -\int_{k\delta}^{(k+1)\delta}\left(\nabla f(y_s) - g_k \right)ds \\
& = \Delta_k - \int_{k\delta}^{(k+1)\delta}\left(\nabla f(y_s) - \nabla f(y_k) + \nabla (y_k)-\nabla f(x_k) + \nabla f(x_k) -g_k\right)ds \\
& = \Delta_k - \delta(\underbrace{\nabla f(y_k) - \nabla f(x_k)}_{=: U_k}) - \int_{k\delta}^{(k+1)\delta}\left(\nabla f(y_s)-\nabla f(y_k)\right) ds+ \delta\underbrace{\left(-\nabla f(x_k) + g_k \right)}_{=:\zeta_k} \\
& = \Delta_k -\delta U_k + \delta \zeta_k -V_k -\underbrace{\sqrt{2}\int_{k\delta}^{(k+1)\delta} \int_{k\delta}^s \nabla^2 f(y_r)dB_r ds}_{=: \delta \Psi_k} = \Delta_k -V_k -\delta(U_k+\Psi_k +\zeta_k). \numberthis \label{eq:deltak1deltakrelation}
\end{align*}
\textbf{Step 4:} Using the above characterization of $\Delta_{k+1}$ in terms of $\Delta_k$ established above, we get
\begin{align*}
\lv \Delta_{k+1} + V_k \rv_2^2 &= \lv \Delta_k -\delta(U_k+\Psi_k+\zeta_k)\rv_2^2 \\
& = \lv \Delta_k \rv_2^2 -2\delta\langle \Delta_k, U_k + \Psi_k + \zeta_k \rangle + \delta^2 \lv U_k + \Psi_k +\zeta_k \rv_2^2.
\end{align*}
Now we take expectation with respect to all sources of randomness (Brownian motion and the randomness in the choice of $S$) conditioned on $x_k,y_k,\tilde{x}$ and $\tilde{y}$ (thus $\Delta_k$ is fixed). Recall that conditioned on $\Delta_k, \tilde{x}$ and $\tilde{y}$, $\Psi_k$ and $\zeta_k$ are zero mean, thus we get
\begin{align*}
\mathbb{E}_{k}\left[\lv\Delta_{k+1} + V_{k}\rv_2^2\right] & = \lv \Delta_k \rv_2^2 -2\delta \underbrace{\langle \nabla f(y_k) - \nabla f(x_k), y_k - x_k \rangle}_{=:\Omega_1} + \delta^2 \underbrace{\mathbb{E}_k \left[\lv U_k+\Psi_k+\zeta_k \rv_2^2 \right]}_{=:\Omega_2},
\end{align*}
where $\mathbb{E}_k\left[\cdot\right]$ denotes conditioning on $x_k$ and $y_k$. First we bound $\Omega_2$
\begin{align*}
\Omega_2 & = \mathbb{E}_k \left[\lv U_k+\Psi_k+\zeta_k \rv_2^2 \right] \\
& = \mathbb{E}_k \left[\lv \nabla f(y_k) - g_k +\Psi_k \rv_2^2 \right] \\
& = \mathbb{E}_k \left[\lv \nabla f(y_k) -  \nabla f(\tilde{x}) - \frac{N}{n}\sum_{i \in S} \left[\nabla f_i(x_k) - \nabla f_i(\tilde{x}) \right] +\Psi_k \rv_2^2 \right]  \\
& = \mathbb{E}_k \Big[\lv \nabla f(y_k) -  \nabla f(x_k) + \nabla f(x_k)-\nabla f(\tilde{y}) - \frac{N}{n}\sum_{i \in S} \left[\nabla f_i(x_k) - \nabla f_i(\tilde{y}) \right]\\ & \qquad \qquad \qquad \qquad \qquad  \qquad + \nabla f(\tilde{y})-\nabla f(\tilde{x})
 - \frac{N}{n}\sum_{i \in S} \left[ \nabla f_i(\tilde{y}) -\nabla f_i(\tilde{x})\right] +\Psi_k \rv_2^2 \Big],
\end{align*}
where in the second equality we used the definition of $U_k$ and $\zeta_k$, while in the third equality we used the definition of $g_k$. By Young's inequality we now have,
\begin{align*}
\Omega_2 & \le \underbrace{4\mathbb{E}_k \left[\lv \nabla f(y_k) -  \nabla f(x_k) \rv_2^2\right]}_{=:\omega_1} + \underbrace{4\mathbb{E}_k \left[\lv\nabla f(x_k)-\nabla f(\tilde{y}) - \frac{N}{n}\sum_{i \in S} \left[\nabla f_i(x_k) - \nabla f_i(\tilde{y}) \right] \rv_2^2\right]}_{=:\omega_2} \\ & \qquad \qquad \qquad \qquad + \underbrace{4\mathbb{E}_k \left[\lv\nabla f(\tilde{y})-\nabla f(\tilde{x})
 - \frac{N}{n}\sum_{i \in S} \left[ \nabla f_i(\tilde{y}) -\nabla f_i(\tilde{x})\right] \rv_2^2\right]}_{=:\omega_3} +\underbrace{4\mathbb{E}_k\left[\lv \Psi_k \rv_2^2 \right]}_{=:\omega_4}. \numberthis \label{eq:omega2masterequation}
\end{align*}
Now we bound each of the 4 terms. First for $\omega_1$ by $M$-smoothness of $f$ we get,
\begin{align*}
\omega_1 = 4\mathbb{E}_k \left[\lv \nabla f(y_k) -  \nabla f(x_k) \rv_2^2\right] \le 4 M \Omega_1.
\end{align*}
Next we upper bound $\omega_4$,
\begin{align}\label{eq:omega4upperbound}
\omega_4 = 4\mathbb{E}_k\left[\lv \Psi_k \rv_2^2 \right] = \frac{8}{\delta^2}\left\lv \int_0^{\delta}(\delta - r) \nabla^2 f(y_r)dB_r \right\rv_2^2 & = \frac{8}{\delta^2}\int_0^{\delta} (\delta - r)^2 \mathbb{E}\left[\lv \nabla^2 f(y_r)\rv_{F}^2 \right]dr \nonumber\\ &\le \frac{8M^2\delta d}{3}. 
\end{align}
Next we will control $\omega_3$. Let us define the random variable $\beta^{(i)} := n(\nabla f(\tilde{y}) - \nabla f(\tilde{x}))/N - \nabla f_i(\tilde{x}) + \nabla f_i(\tilde{y})$. Observe that $\beta$ is zero mean (taking expectation over choice of $i$). Thus we have,
\begin{align*}
\omega_3 & = 4\mathbb{E}_k \left[\lv\nabla f(\tilde{y})-\nabla f(\tilde{x})
 - \frac{N}{n}\sum_{i \in S} \left[ \nabla f_i(\tilde{y}) -\nabla f_i(\tilde{x})\right] \rv_2^2\right]  = 4 \mathbb{E}_k \left[\lv \frac{N}{n} \sum_{i=1}^n \beta^{(i)} \rv_2^2\right] \\
 & \overset{(i)}{\le} \frac{4N^2}{n^2} \sum_{i=1}^n \mathbb{E}_k\left[ \lv \beta^{(i)} \rv_2^2 \right]  = \frac{4 N^2}{n} \mathbb{E}_k \left[\lv \beta \rv_2^2 \right] \overset{(ii)}{\le} \frac{4N}{n} \left[\sum_{i=1}^N \mathbb{E}_k\left[\lv \nabla  f_i(\tilde{x}) -  \nabla f_i(\tilde{y}) \rv_2^2 \right]\right] \\
 & \overset{(iii)}{\le} \frac{8M}{n}\mathcal{D}_f(\tilde{x},\tilde{y}), \numberthis \label{eq:omega3independence}
\end{align*}
where $(i)$ follows as $\beta^{(i)}$ are zero-mean and independent random variables, $(ii)$ follows by the fact that for any random variable $R$, $\mathbb{E}\left[\lv R- \mathbb{E} R \rv_2^2\right] \le \mathbb{E}\left[ \lv R \rv_2^2\right]$ and by the fact that $\mathbb{E}\left[\nabla f_i(\tilde{y}) - \nabla f_i(\tilde{x})\right] = \sum_{i=1}^N \nabla f_i(\tilde{y}-\nabla f_i(\tilde{x}))/N$. Finally $(iii)$ follows by using the $\tilde{M}$ ($M/N$) smoothness of each $f_i$. Finally we bound $\omega_2$
\begin{align*}
\omega_2 & = 4\mathbb{E}_k \left[\lv\nabla f(x_k)-\nabla f(\tilde{y}) - \frac{N}{n}\sum_{i \in S} \left[\nabla f_i(x_k) - \nabla f_i(\tilde{y}) \right] \rv_2^2\right] \\
& \overset{(i)}{\le}  8\mathbb{E}_k \left[\lv\nabla f(x_k)-\nabla f(y_k) - \frac{N}{n}\sum_{i \in S} \left[\nabla f_i(x_k) - \nabla f_i(y_k) \right] \rv_2^2\right] \\
& + 8\mathbb{E}_k \left[\lv\nabla f(y_k)-\nabla f(\tilde{y}) - \frac{N}{n}\sum_{i \in S} \left[\nabla f_i(y_k) - \nabla f_i(\tilde{y}) \right] \rv_2^2\right] \\
& \overset{(ii)}{\le} \frac{8M}{n}\Omega_1 + \frac{8N}{n} \left[\sum_{i=1}^N \underbrace{\mathbb{E}_k\left[\lv \nabla  f_i(\tilde{y}) -  \nabla f_i(y_k) \rv_2^2\right]}_{=:\tilde{\omega}_{2,i}}\right], \numberthis \label{eq:omega2masterbound}
\end{align*}
where $(i)$ follows by Young's inequality and $(ii)$ is by the same techniques used to control $\omega_3$ applied to the two terms. We will now control $\tilde{\omega}_{2,i}$ using a techniques introduced by \cite{dubey2016variance},
\begin{align*}
\tilde{\omega}_{2,i} & = \mathbb{E}_k\left[\lv \nabla  f_i(\tilde{y}) -  \nabla f_i(y_k) \rv_2^2 \right]  \overset{(i)}{\le} \frac{M^2}{N^2}\mathbb{E}_k\left[ \lv \tilde{y} - y_k \rv_2^2 \right] \\ &\overset{(ii)}{=} \frac{M^2}{N^2} \mathbb{E}_k\left[\lv -\int_{u\delta}^{k\delta} \nabla f(y_s) ds +\sqrt{2}\int_{u\delta}^{k\delta} dB_r \rv_2^2 \right]\\
& \overset{(iii)}{\le} \frac{2M^2}{N^2} \mathbb{E}_k\left[\lv \int_{u\delta}^{k\delta} \nabla f(y_s) ds \rv_2^2 \right] + \frac{4M^2}{N^2} \mathbb{E}_k\left[\lv \int_{u\delta}^{k\delta} dB_r \rv_2^2 \right]\\ & \overset{(iv)}{\le} \frac{2M^2(k-u)\delta}{N^2} \int_{u\delta}^{k\delta} \mathbb{E}_{y_s \sim p^*}\left[\lv \nabla f(y_s) \rv_2^2\right] ds + \frac{4M^2\delta d (k-u)}{N^2} \\
& \overset{(v)}{\le} \frac{2M^3 (k-u)^2 \delta^2 d}{N^2} + \frac{4M^2 \delta d (k-u)}{N^2}\overset{(vi)}{\le} \frac{2M^3 \tau^2 \delta^2 d}{N^2} + \frac{4M^2 \delta d \tau}{N^2},
%& \overset{(i)}{=} \sum_{j=-\tau}^{0} \mathbb{E}_k \left[\lv \nabla f_i(\tilde{y}) - \nabla f_i(y_k) \rv_2^2 \Big \lvert \tilde{y}= y_j\right] \mathbb{P}\left[\tilde{y} = y_j\right] \\
%& = \sum_{j=-\tau}^{0} \mathbb{E}_k \left[\lv \nabla f_i(y_j) - \nabla f_i(y_k) \rv_2^2 \Big \lvert \tilde{y}= y_j\right] \mathbb{P}\left[\tilde{y} = y_j\right] \\
%& \overset{(ii)}{\le} \frac{M^2}{N^2} \left(\sum_{j=0}^{k-1}\mathbb{E}_k \left[\lv \int_{j\delta}^{k\delta} \nabla f(y_s) - \sqrt{2}\int_{j\delta}^{k\delta} dB_r \rv_2^2 \right] (1-p)^{k-j-1}p \right) \\
%& \overset{(iii)}{\le}  \frac{M^2}{N^2} \left(\sum_{j=0}^k \left[2\delta^2 \mathbb{E}_{y \sim p^*}\left[\lv \nabla f(y)\rv_2^2 \right] + 4\delta d(k-j)\right](1-p)^{ki-j-1}p \right) \\
%& \le \frac{2M^2\delta^2}{N^2} \mathbb{E}_{y \sim p^*}\left[\lv \nabla f(y)\rv_2^2 \right] \sum_{j=1}^k j^2(1-p)^{j-1} + \frac{4 pd\delta M^2}{N^2} \sum_{j=1}^k j (1-p)^{j-1} \\
%& \le \frac{2\delta^2 M^2}{N^2p^2}\mathbb{E}_{y \sim p^*}\left[\lv \nabla f(y)\rv_2^2 \right] + \frac{4d\delta M^2}{N^2p}\\
%& \overset{(iv)}{\le} \frac{8\delta^2  M^3 d}{n^2} + \frac{8d\delta  M^2}{nN}, \qquad (p = )
\end{align*}
where $(i)$ follows by $M/N$-smoothness of $f_i$, $(ii)$ follows as $\tilde{y} = y_{u\delta}$, $(iii)$ follows by Young's inequality, $(iv)$ follows by Jensen's inequality, $(v)$ follows by Lemma 3 in \citep{dalalyan2017further} to bound $\mathbb{E}_{y \sim p^*}\left[\lv \nabla f(y)\rv_2^2 \right]  \le Md$ and finally $(vi)$ is by the upper bound $k-u \le \tau$ (the epoch length). Plugging this bound into \eqref{eq:omega2masterbound} we get
\begin{align*}
\omega_2 & \le  \frac{8M}{n}\Omega_1 +\frac{16\delta^2 M^3 d \tau^2 }{n } + \frac{32d \delta M^2 \tau }{n}.
\end{align*}
By the bounds we have established on $\omega_1,\omega_2,\omega_3$ and $\omega_4$ we get that $\Omega_2$ is upper bounded by,
\begin{align*}
\Omega_2 & \le 4M\left( 1+\frac{2}{n}\right) \Omega_1+ \frac{8M}{n}\mathcal{D}_f(\tilde{x},\tilde{y}) + \frac{8M^2\delta d}{3}+\frac{16\delta^2 M^3 d \tau^2 }{n } + \frac{32d \delta M^2 \tau }{n}.
\end{align*}
Next we control $\Omega_1$ by using the convexity of $f$,
\begin{align} \label{eq:Omega1control}
\Omega_1 = \langle \nabla f(y_k) - \nabla f(x_k), y_k - x_k \rangle & \ge D_f(x_k,y_k) + \frac{m}{2}\lv \Delta_k \rv_2^2,
\end{align}
where $\mathcal{D}_f(a,b):= f(a) - f(b) - \langle \nabla f(b),a-b \rangle$ for any $a,b \in \mathbb{R}^d$, is the Bregman divergence of $f$.
Coupled with the bound on $\Omega_2$ we now get that,
\begin{align}
\mathbb{E}_{k}\left[\lv\Delta_{k+1} + V_{k}\rv_2^2\right] & \le \left( 1-\delta m \left(1-2\delta M\left(1+\frac{2}{n}\right)\right)  \right)\lv \Delta_k \rv_2^2 - 2 \delta \left(1-2\delta M\left(1+\frac{2}{n}\right)\right)  \mathcal{D}_f(x_k,y_k)\\& + \delta^2\left(\frac{8M}{n}\mathcal{D}_f(\tilde{x},\tilde{y}) + \frac{8M^2\delta d}{3}+\frac{16\delta^2 M^3 d \tau^2 }{n } + \frac{32d \delta M^2 \tau }{n}\right). \label{eq:masterbounddeltak1}
\end{align}
\textbf{Step 5:} Substituting the bound on $\mathbb{E}_{k}\left[\lv\Delta_{k+1} + V_{k}\rv_2^2\right] $ and $\mathbb{E}\left[ \lv V_k \rv_2^2\right]$ into \eqref{eq:mastersvrgcontrol1} we get
\begin{align*}
\mathbb{E}_k\left[\lv \Delta_{k+1} \rv_2^2 \right] & \le (1+a) \Bigg( \left( 1-\delta m \left(1-2\delta M\left(1+\frac{2}{n}\right)\right)  \right)\lv \Delta_k \rv_2^2 - 2 \delta \left(1-2\delta M\left(1+\frac{2}{n}\right)\right)  \mathcal{D}_f(x_k,y_k)\\& + \delta^2\left(\frac{8M}{n}\mathcal{D}_f(\tilde{x},\tilde{y}) + \frac{8M^2\delta d}{3}+\frac{16\delta^2 M^3 d \tau^2 }{n} + \frac{32d \delta M^2 \tau }{n}\right) \Bigg) \\&  + \left(1+ \frac{1}{a} \right)\frac{\delta^4}{2}\left(L^2d^2 + M^3 d \right) \\
& = (1+a) \left( 1-\delta m \left(1-2\delta M\left(1+\frac{2}{n}\right)\right)  \right) \lv \Delta_k \rv_2^2 \\
& -2(1+a)\delta\left(1-2\delta M\left(1+\frac{2}{n}\right)\right)\mathcal{D}_f(x_k,y_k)\\ 
& + 8(1+a) \frac{M\delta^2}{n}\mathcal{D}_f(\tilde{x},\tilde{y})+ (1+a) \delta^3\underbrace{\left(\frac{8M^2d}{3}+ \frac{32d  M^2 \tau }{n}+\frac{8 M^3 d \tau^2  \delta }{n}\right)}_{=:\circ} \\
&+ \frac{\delta^4}{2}\left(1+ \frac{1}{a} \right)\underbrace{\left(L^2d^2 + M^3d \right)}_{=:\square} .
\end{align*}
Define the contraction rate to be $\alpha: = \delta m \left(1-2\delta M\left(1+\frac{2}{n}\right)\right) $, then we have
\begin{align*}
{E}_k\left[\lv \Delta_{k+1} \rv_2^2 \right] \leq & (1+a)(1-\alpha)\lv \Delta_k \rv_2^2 -2\frac{\alpha(1+a) }{m} \mathcal{D}_f(x_k,y_k)+ 8(1+a) \frac{M\delta^2}{n}\mathcal{D}_f(\tilde{x},\tilde{y}) \\
& + (1+a) \delta^3\circ + \frac{\delta^4}{2}\left(1+ \frac{1}{a} \right)\square,
\end{align*}
and with $a =\alpha/(1- \alpha)$ then we have
\begin{align*}
\mathbb{E}_k\left[\lv \Delta_{k+1} \rv_2^2 \right] & \le \lv \Delta_k \rv_2^2 -\frac{2\alpha}{m(1-\alpha)}\mathcal{D}_f(x_k,y_k) + \frac{8M\delta^2}{(1-\alpha)n}\mathcal{D}_f(\tilde{x},\tilde{y})+ \frac{\delta^3\circ}{1-\alpha}  + \frac{\delta^4}{2}\frac{\square}{\alpha}.
\end{align*}
We now sum this inequality from $k=u\tau$ to $(u+1)\tau-1$ we get,
\begin{align}
\mathbb{E}\left[ \lv \Delta_{(u+1)\tau} \rv_2^2 \right] \le \lv \Delta_{u\tau} \rv_2^2 -  \frac{2\alpha}{m(1-\alpha)} \sum_{k=1}^{\tau} \mathcal{D}_f(x_k,y_k) + \frac{8M\delta^2\tau}{(1-\alpha)n}\mathcal{D}_f(\tilde{x},\tilde{y}) + \frac{\delta^3\tau \circ}{1-\alpha } + \frac{\delta^4\tau}{2}\frac{\square}{\alpha}. \label{eq:finalcountdown}
\end{align}
Using the strong convexity and smoothness of $f$ we have,
\begin{align*}
\frac{m}{2}\lv x-y\rv_2^2 \le \mathcal{D}_f(x,y) \le \frac{M}{2}\lv x-y\rv_2^2, \qquad \forall x,y\in \mathbb{R}^d.
\end{align*}
Using this in \eqref{eq:finalcountdown} and rearranging terms we get,
\begin{align*}
\frac{1}{\tau}\sum_{k=u\tau}^{(u+1)\tau-1}\mathcal{D}_f(x_k,y_k)  &\le \left( \frac{1-\alpha}{\alpha \tau }+ \frac{4mM\delta^2}{\alpha n} \right) \mathcal{D}_f(\tilde{x},\tilde{y}) + \frac{\delta^3 m\circ}{2\alpha} + \frac{(1-\alpha)m\delta^4\square}{4\alpha^2},
\end{align*}
%Using this in \eqref{eq:finalcountdown} and rearranging terms we get,
%\begin{align*}
%\frac{1}{\tau}\sum_{k=u\tau}^{(u+1)\tau-1} \lv \Delta_k  \rv_2^2 &\le \frac{n(1-\alpha)}{\delta m(n-4\delta M)}\left(\frac{1}{\tau}+ \frac{4M^3\delta^2}{(1-\alpha)n} \right) \lv \Delta_{u\tau} \rv_2^2 + \frac{n\delta^3 \circ}{\delta m (n-4\delta M)} + \frac{n (1-\alpha)\delta^4}{2\delta m(n-4\delta M)}\left(\frac{\square}{\alpha}+ \frac{\triangle}{1-\alpha}\right),
%\end{align*}
where we used the fact that $x_{u\delta} = \tilde{x}$ and $y_{u\delta} = \tilde{y}$. Now the $\tilde{x}^{u+1}$ and $\tilde{y}^{u+1}$ are chosen uniformly at random from $x_{u\tau},\ldots,x_{(u+1)\tau-1}$ and $y_{u\tau},\ldots,y_{(u+1)\tau-1}$.
%Using the fact that $\tilde{x}^u$ ($\tilde{x}$) and $\tilde{y}^u$ ($\tilde{y}$) are optimally coupled by taking expectations we get that,
\begin{align*}
\E{\mathcal{D}_f(\tilde{x}^{u+1},\tilde{y}^{u+1})} & \le \left( \frac{1-\alpha}{\alpha \tau }+ \frac{4mM\delta^2}{\alpha n} \right)\E{\mathcal{D}_f(\tilde{x}^{u},\tilde{y}^{u})}+ \frac{\delta^3 m\circ}{2\alpha} + \frac{(1-\alpha)m\delta^4\square}{4\alpha^2},\end{align*}
for any $u \in \{0,\ldots,\floor*{T/\tau}\}$.  Unrolling the equation above for $\floor*{T/\tau}$ steps we get,
\begin{align*}
\E{\mathcal{D}_f(\tilde{x}^{\floor*{T/\tau}\tau},\tilde{y}^{\floor*{T/\tau}\tau})} & \le (1-\rho)^{\floor*{T/\tau}}\E{\mathcal{D}_f(\tilde{x}^{0},\tilde{y}^{0})} +\frac{\delta^3 m\circ}{2\alpha\rho} + \frac{(1-\alpha)m\delta^4\square}{4\alpha^2 \rho},
\end{align*}
where we denote by $\rho=1-\frac{1-\alpha}{\alpha \tau }- \frac{4mM\delta^2}{\alpha n} $.
%\begin{align}
%W_2(p^{\floor*{T/\tau}\tau},p^*) & \le \left[\frac{n(1-\alpha)}{\delta m(n-4\delta M)}\left(\frac{1}{\tau}+ \frac{4M^3\delta^2}{(1-\alpha)n} \right)\right]^{\floor*{T/\tau}/2}W_2(p^{0},p^*) + \sqrt{\frac{\frac{n\delta^3 \circ}{\delta m (n-4\delta M)} + \frac{n (1-\alpha)\delta^4}{2\delta m(n-4\delta M)}\left(\frac{\square}{\alpha}+ \frac{\triangle}{1-\alpha}\right)}{1-\frac{n(1-\alpha)}{\delta m(n-4\delta M)}\left(\frac{1}{\tau}+ \frac{4M^3\delta^2}{(1-\alpha)n} \right)}}. \label{eq:masterwassersteinsvrgbound}
%\end{align}
Finally we use again the strong convexity and smoothness of $f$ to obtain
\begin{align*}
 \mathbb{E}\left[ \lv {\tilde \Delta}^{\floor*{T/\tau}\tau} \lv_2^2\right]& \le (1-\rho)^{\floor*{T/\tau}}\frac{M}{m}
 \E{\lv {\tilde \Delta}^0 \lv_2^2} +\frac{\delta^3 \circ}{\alpha\rho} + \frac{(1-\alpha)\delta^4\square}{2\alpha^2 \rho}.
\end{align*}
 Using the fact that $\tilde{x}^u$ ($\tilde{x}$) and $\tilde{y}^u$ ($\tilde{y}$) are optimally coupled by taking expectations we get that,
\begin{align}
W^2_2(p^{\floor*{T/\tau}\tau},p^*) & \le (1-\rho)^{\floor*{T/\tau}}\frac{M}{m} W^2_2(p^{(0)},p^*) + {\frac{\delta^3 \circ}{\alpha\rho} + \frac{(1-\alpha)\delta^4\square}{2\alpha^2\rho}}. \label{eq:masterwassersteinsvrgbound}
\end{align}
Now for $n>2$, and $\delta<\frac{1}{8M}$ we have $\alpha\geq \delta m/2$, and $ \frac{4mM\delta^2}{\alpha n}\leq \frac{8M\delta}{n}\leq 1/n\leq 1/4$. Then consider $\tau=4/\alpha$, in order to have $\rho=1/2$.
\begin{align*}
\frac{\delta^3 \circ}{\alpha\rho}
\leq 32 \delta^2 M^2d \left(\frac{1}{3}+ \frac{4 \tau }{n}+\frac{ M  \tau^2  \delta }{n}\right) \leq 32 \delta^2 M^2d \left(\frac{1}{3}+ \frac{32 }{n \delta m}+\frac{ 64 M     }{ \delta m^2 n}\right) \leq 4096 \frac{\delta d M^3}{m^2 n},
\end{align*}
and
\begin{align*}
\frac{\delta^4\square}{2\alpha^2\rho}\leq \frac{4\delta^2}{m^2}\left(L^2d^2 + M^3d \right).
\end{align*}
Therefore
\[
W^2_2(p^{\floor*{T/\tau}\tau},p^*)  \le \frac{1}{2}^{\floor*{T/\tau}}\frac{M}{m} W^2_2(p^{(0)},p^*) + 4096 \frac{\delta d M^3}{m^2 n}+ \frac{4\delta^2 L^2d^2}{m^2} +\frac{4\delta^2M^3d}{m^2}.
\]
Finally our choice of $n>2$, and $\delta<\frac{1}{8M}$ and $ \tau=4/\alpha$  ensures that
\begin{align*}
W^2_2(p^T,p^*)  \le \exp\left( -\frac{\delta m t }{28}\right) \frac{M}{m} W^2_2(p^{(0)},p^*) + 4096 \frac{\delta d M^3}{m^2 n}+ \frac{4\delta^2 L^2d^2}{m^2} +\frac{4\delta^2M^3d}{m^2},
\end{align*}
for all $T>0$ such that $T$ mod $\tau=0$, which completes the proof of part 1.

\textbf{Proof for Option 2:} To prove part 2 of the theorem Steps 1-3 are same as above. The technique to control $\Omega_2$ is going to differ which leads to a different bound.

\textbf{Step 4:} As before we have
\begin{align*}
\lv \Delta_{k+1} + V_k \rv_2^2 &= \lv \Delta_k -\delta(U_k+\Psi_k+\zeta_k)\rv_2^2 \\
& = \lv \Delta_k \rv_2^2 -2\delta\langle \Delta_k, U_k + \Psi_k + \zeta_k \rangle + \delta^2 \lv U_k + \Psi_k +\zeta_k \rv_2^2.
\end{align*}
Now we take expectation with respect to all sources of randomness (Brownian motion and the randomness in the choice of $S$) conditioned on $x_k,y_k,\tilde{x}$ and $\tilde{y}$ (thus $\Delta_k$ is fixed). Recall that conditioned on $\Delta_k, \tilde{x}$ and $\tilde{y}$, $\Psi_k$ and $\zeta_k$ are zero mean, thus we get
\begin{align*}
\mathbb{E}_{k}\left[\lv\Delta_{k+1} + V_{k}\rv_2^2\right] & = \lv \Delta_k \rv_2^2 -2\delta \underbrace{\langle \nabla f(y_k) - \nabla f(x_k), y_k - x_k \rangle}_{=:\Omega_1} + \delta^2 \underbrace{\mathbb{E}_k \left[\lv U_k+\Psi_k+\zeta_k \rv_2^2 \right]}_{=:\Omega_2},
\end{align*}
where $\mathbb{E}_k\left[\cdot\right]$ denotes conditioning on $x_k$ and $y_k$. First we bound $\Omega_2$
\begin{align*}
\Omega_2 & = \mathbb{E}_k \left[\lv U_k+\Psi_k+\zeta_k \rv_2^2 \right] \\
& = \mathbb{E}_k \left[\lv \nabla f(y_k) - g_k +\Psi_k \rv_2^2 \right] \\
& = \mathbb{E}_k \left[\lv \nabla f(y_k) -  \nabla f(\tilde{x}) - \frac{N}{n}\sum_{i \in S} \left[\nabla f_i(x_k) - \nabla f_i(\tilde{x}) \right] +\Psi_k \rv_2^2 \right] \\
& = \mathbb{E}_k \left[\lv \nabla f(y_k) -  \nabla f(x_k)+\nabla f(x_k)-\nabla f(\tilde{x}) - \frac{N}{n}\sum_{i \in S} \left[\nabla f_i(x_k) - \nabla f_i(\tilde{x}) \right] +\Psi_k \rv_2^2 \right] \\
& \le \underbrace{3\mathbb{E}_k \left[\lv \nabla f(y_k) -  \nabla f(x_k) \rv_2^2 \right]}_{=:\omega_1} + \underbrace{3\mathbb{E}_k \left[\lv \nabla f(x_k)-\nabla f(\tilde{x}) - \frac{N}{n}\sum_{i \in S} \left[\nabla f_i(x_k) - \nabla f_i(\tilde{x}) \right]\rv_2^2 \right]}_{=:\omega_2} \\ &\qquad \qquad \qquad \qquad \qquad \qquad \qquad \qquad \qquad \qquad \qquad \qquad \qquad \qquad +\underbrace{3\mathbb{E}_k \left[\lv \Psi_k \rv_2^2 \right]}_{=:\omega_3},
\end{align*}
where the last step is by Young's inequality. First we claim a bound on $\omega_3$,
\begin{align*}
\omega_3 = 3\mathbb{E}_k \left[\lv \Psi_k \rv_2^2 \right] \le 2M^2\delta d,
\end{align*}
by the same argument as used in \eqref{eq:omega4upperbound}. Next we control $\omega_1$ by
\begin{align*}
\omega_1 = 3\mathbb{E}_k \left[\lv \nabla f(y_k) -  \nabla f(x_k) \rv_2^2 \right] & \le 3M\Omega_1,
\end{align*}
by the $M$-smoothness of $f$. Finally we control $\omega_2$,
\begin{align*}
\omega_2 = 3\mathbb{E}_k \left[\lv \nabla f(x_k)-\nabla f(\tilde{x}) - \frac{N}{n}\sum_{i \in S} \left[\nabla f_i(x_k) - \nabla f_i(\tilde{x}) \right]\rv_2^2 \right] & \le \frac{M^2}{n} \mathbb{E}\left[ \lv x_k - \tilde{x}\rv_2^2\right],
\end{align*}
where the last inequality follows by arguments similar to \eqref{eq:omega3independence}. By the definition of $\tilde{x}$ we get
\begin{align*}
\mathbb{E}\left[\Vert x_k-\tilde x\Vert_2^2\right]
\leq
\mathbb{E}\left[ \Vert \sum_{j=\tau s}^{k-1} (x_{j+1}-x_j) \Vert_2^2\right] &\leq (k-s\tau) \mathbb{E}\left[\sum_{j=\tau s}^{k-1} \Vert  x_{j+1}-x_j \Vert_2^2\right]\\ &\leq \tau  \sum_{j=\tau s}^{k-1}\mathbb{E}\left[ \Vert  x_{j+1}-x_j \Vert_2^2 \right],
\end{align*}
where the first inequality follows by Jensen's inequality. Further we have,
\begin{align*}
 \mathbb{E}\left[\Vert  x_{j+1}-x_j \Vert^2  \right]
&= \mathbb{E}\left[ \Vert \delta g_j -\sqrt{2\delta}\xi_j\Vert_2^2\right] \\
 & =\mathbb{E}\left[\Vert \sqrt{2\delta}\xi_j+\delta\left(\nabla f(\tilde{x}) - \frac{N}{n}\sum_{i \in S} \left[\nabla f_i(x_k) - \nabla f_i(\tilde{x}) \right]\right)\Vert_2^2\right] \\
&\overset{(i)}{\leq} 8\delta d  +4 \delta^2 \mathbb{E}\left[\Vert \nabla f (y_k) \Vert_2^2\right] +4\delta^2\mathbb{E}\left[\Vert \nabla f (y_k) -\nabla f(x_k)\Vert_2^2\right] \\
& \qquad \qquad +4\delta^2 \mathbb{E}\left[\Vert  -\nabla f(x_k)+ \nabla f(\tilde{x}) - \frac{N}{n}\sum_{i \in S} \left[\nabla f_i(x_k) - \nabla f_i(\tilde{x}) \right]\Vert^2\right] \\
&\overset{(ii)}{\le} 8 \delta d   + 4M\delta^2 d +4\delta^2M\Omega_1+ 4\frac{M^2\delta^2}{n} \mathbb{E}\left[\Vert x_j-\tilde x\Vert^2\right],
\end{align*}
where $(i)$ follows by Young's inequality and $(ii)$ follows by the bound of $\mathbb{E}\left[\Vert \nabla f (y_k) \Vert_2^2\right] \le Md$, the $M$-smoothness of $f$. Let us define $\spadesuit : = \tau^2(8\delta d + 4M\delta^2d+ 4\delta^2M \Omega_1)$ and $\rho: = 4\tau M^2\delta^2/n$. Coupled with the bound above we get that,
\begin{align*}
\mathbb{E}\left[\lv x_k - \tilde{x}\rv_2^2  \right] & \le \spadesuit + \rho \sum_{j=\tau s}^{k-1} \mathbb{E}\left[\lv x_{j+1} - x_j \rv_2^2\right],
\end{align*}
by using discrete Gr\"{o}nwall lemma \citep[see, e.g.,][]{clark1987short} we get,
\begin{align*}
\mathbb{E}\left[\lv x_k - \tilde{x}\rv_2^2  \right] &\le \spadesuit \exp(\tau \rho).
\end{align*}
Combined with the bound on $\omega_1$ and $\omega_3$ this yields a bound on $\Omega_2$ which is
\begin{align*}
\Omega_2 & \le 3M\Omega_1 + 2M^2\delta d +\frac{ \tau^2 M^2}{n}(8\delta d + 4M\delta^2d+ 4\delta^2 M \Omega_1)\exp\left(\frac{4\tau^2 M^2\delta^2}{n}\right)\\
&\le 3M\Omega_1\underbrace{\left[ 1+\frac{4\tau^2M^2\delta^2}{3n} \exp\left(\frac{4\tau^2 M^2\delta^2}{n}\right)\right]}_{=:\square}+2\delta\underbrace{\left[ M^2d +\frac{2d \tau^2M^2}{n}(2+ M\delta )\exp\left(\frac{4\tau^2 M^2\delta^2}{n}\right) \right]}_{=:\triangle}.
\end{align*}
As before, using strong convexity of $f$, we get that $\Omega_1$ is bounded by
\begin{align*}
\Omega_1 = \langle \nabla f(y_k) - \nabla f(x_k), y_k - x_k \rangle & \ge {m}\lv \Delta_k \rv_2^2.
\end{align*}
 Having established these bounds on $\Omega_1$ and $\Omega_2$ we get,
\begin{align*}
\mathbb{E}_{k}\left[\lv\Delta_{k+1} + V_{k}\rv_2^2\right] & \le \lv \Delta_k \rv_2^2 -2\Omega_1 \delta\left(1-\frac{3M\square\delta}{2}\right) + 2\delta^3\triangle\\
& \le  \left(1-2\delta m  \left(1-\frac{3M\square\delta}{2}\right)\right) \lv \Delta_k \rv_2^2 +2\delta^3\triangle .
\end{align*}
\textbf{Step 5:}  Substituting the bound on $\mathbb{E}_{k}\left[\lv\Delta_{k+1} + V_{k}\rv_2^2\right] $ and $\mathbb{E}\left[ \lv V_k \rv_2^2\right]$ into \eqref{eq:mastersvrgcontrol} we get for $\circ:= L^2d^2 +M^3d$
\begin{align*}
\mathbb{E}_k\left[\lv \Delta_{k+1} \rv_2^2 \right]
 & \le (1+a) \left( \left(1-2\delta m  \left(1-\frac{3M\square\delta}{2}\right)\right) \lv \Delta_k \rv_2^2 +2\delta^3\triangle\right) + \left(1+ \frac{1}{a} \right)\frac{\delta^4 \circ}{2} \\
& \le  (1+a) \left(1-2\delta m  \left(1-\frac{3M\square\delta}{2}\right)\right) \lv \Delta_k \rv_2^2
+2(1+a)\delta^3\triangle + \left(1+ \frac{1}{a} \right)\frac{\delta^4 \circ}{2} .
\end{align*}
Define the contraction rate $\alpha : =\delta m  \left(1-\frac{3M\square\delta}{2}\right)$. Further we choose $a = \alpha/(1-2\alpha)$, then we get,
\begin{align*}
\mathbb{E}_k\left[\lv \Delta_{k+1} \rv_2^2 \right] & \le \left(1-\alpha \right)\lv \Delta_k \rv_2^2 +\frac{2(1-\alpha)\delta^3\triangle}{1-2\alpha}+\frac{(1-\alpha)\delta^4\circ}{2\alpha}.
\end{align*}
Taking the global expectation, we obtain by a direct expansion:
\begin{align*}
\mathbb{E}\left[\lv \Delta_{k} \rv_2^2 \right] & \le \left(1-\alpha \right)^k \E{\lv \Delta_0 \rv_2^2} +\frac{2(1-\alpha)\delta^3\triangle}{(1-2\alpha)\alpha}+\frac{(1-\alpha)\delta^4\circ}{2\alpha^2}.
\end{align*}
Let us use the fact that $\tau < \sqrt{n}/(2\tau M)$ then we have $\exp(4\tau^2 M^2\delta^2/n)<3$, $\square \leq4/3$ and $\alpha\geq \delta m(1-2M\delta)$. We assume also that $\delta\leq 1/(4M)$ then $\alpha\geq \delta m/2$ and
\begin{align*}
\mathbb{E}\left[\lv \Delta_{k} \rv_2^2 \right] & \le \left(1-\frac{\delta m }{2} \right)^k \E{\lv \Delta_0 \rv_2^2} +\frac{8\delta^2\triangle}{m}+\frac{2\delta^2\circ}{m^2}.
\end{align*}
Using that $\triangle \leq M^2d\left(1 +\frac{9\tau^2}{n}\right)$ and $\circ\leq L^2d^2+M^3d $ we finally obtain
\begin{align*}
\mathbb{E}\left[\lv \Delta_{k} \rv_2^2 \right] & \le \left(1-\frac{\delta m }{2} \right)^k \E{\lv \Delta_0 \rv_2^2} +\frac{8\delta^2 M^2d}{m}\left(1+\frac{M}{4m} +\frac{9\tau^2}{n}\right)+\frac{2\delta^2 L^2d^2}{m^2}.
\end{align*}
The result follows.
\end{proof}

\subsection{SAGA Proof}
%!TEX root = main.tex
\label{app:sagaproof}The proof of SAGA for Langevin diffusion largely mirrors the proof of Theorem \ref{thm:svrgtheorem} with two key differences. One is we use a different Lyapunov function, specifically we use the Lyapunov function well studied in the optimization literature to analyze SAGA for variance reduction in optimization introduced by \cite{hofmann2015variance} and subsequently simplified by \cite{defazio2016simple}. Secondly, we in Algorithm \ref{alg:saga} we do not have access to the full gradient at every step, this makes it difficult to analyze some terms in the proof (specifically the term analogous to $\tilde{\omega}_{2,i}$ in the proof of Theorem \ref{thm:svrgtheorem}); we handle this difficultly by borrowing a neat trick developed by \cite{dubey2016variance}. Another way to control this term that leads to a much simpler proof is the method developed above in the second part (Step 4) of the proof of Theorem \ref{thm:svrgtheorem}. However the constants obtained by using the simpler techniques are much worse.
\begin{proof}[Proof of Theorem \ref{thm:mainsaga}]
We will proceed as in the proof of Theorem \ref{thm:svrgtheorem} and borrow notation established in the proof of Theorem \ref{thm:svrgtheorem}. For $k \in \{1,\ldots,T \}$,  we denote by $\{h_k^i\}_{i=0}^N$ analogously to $\{g_k^i\}_{i=0}^N$ (updated at the same point in the sequence $\{y_k \}$). We consider the Lyapunov function $T_k= c\sum_{i}^N\Vert g_{k}^i-h_{k}^i\Vert_2^2 +\Vert x_k-y_k\Vert_2^2$ for some constant $c>0$ to that will be chosen later. The first 4 steps of the proof are exactly the same as the proof above, in Step 5 we shall control the norm of the other part of the Lyapunov function involving $g_k^i$ and $h_k^i$. In the rest of the steps we gather these bounds on the different parts of the Lyapunov function and establish convergence.

\textbf{Step 1:} By Young's inequality we have that $\forall a>0$,
\begin{align} \label{eq:mastersvrgcontrol}
\lv \Delta_{k+1} \rv_2^2 & = \lv \Delta_{k+1} +V_k - V_k\rv_2^2 \le \left(1+ a\right)\lv \Delta_{k+1} + V_k \rv_2^2 + \left(1+\frac{1}{a} \right) \lv V_k \rv_2^2.
\end{align}
We will choose $a$ at a later stage in the proof to minimize the bound on the right hand side.

\textbf{Step 2:} By Lemma 6 in \citep{dalalyan2017user} we have the bound,
\begin{align}
\mathbb{E}\left[\lv V_k \rv_2^2 \right] \le \left(\frac{\delta^2 L d}{2}+ \frac{\delta^2 M^{3/2}\sqrt{d}}{2} \right)^2 \le \frac{\delta^4}{2}\left(L^2d^2 + M^3 d \right).\label{eq:Vbound1}
\end{align}
\textbf{Step 3:} Next we will bound the other term in \eqref{eq:mastersvrgcontrol}. First we express $\Delta_{k+1}$ in terms of $\Delta_k$,
\begin{align*}
\Delta_{k+1} &= \Delta_{k} + (y_{k+1}- y_k) + (x_{k+1}-x_k) \\
& = \Delta_k + \left(-\int_{\delta k}^{\delta(k+1)}\nabla f(y_s) ds + \sqrt{2\delta}\xi_k \right) + \left(-\int_{\delta k}^{\delta(k+1)} g(x_k)ds + \sqrt{2\delta}\xi_k \right) \\
& = \Delta_k -\int_{\delta k}^{\delta(k+1)}\left(\nabla f(y_s) - g(x_k) \right)ds \\
& = \Delta_k - \int_{\delta k}^{\delta(k+1)}\left(\nabla f(y_s) - \nabla f(y_k) + \nabla (y_k)-\nabla f(x_k) + \nabla f(x_k) -g(x_k)\right)ds \\
& = \Delta_k - \delta(\underbrace{\nabla f(y_k) - \nabla f(x_k)}_{=: U_k}) - \int_{\delta k}^{\delta(k+1)}\left(\nabla f(y_s)-\nabla f(y_k)\right) ds+ \delta\underbrace{\left(-\nabla f(x_k) + g(x_k) \right)}_{=:\zeta_k} \\
& = \Delta_k -\delta U_k + \delta \zeta_k -V_k -\sqrt{2}\underbrace{\int_{k\delta}^{(k+1)\delta} \int_{k\delta}^s \nabla^2 f(y_r)dB_r ds}_{=:\delta \Psi_k} = \Delta_k -V_k -\delta(U_k+\Psi_k+\zeta_k).
\end{align*}
\textbf{Step 4:} Using the above characterization of $\Delta_{k+1}$ in terms of $\Delta_k$, we now get
\begin{align*}
\lv \Delta_{k+1} + V_k \rv_2^2 &= \lv \Delta_k -\delta(U_k+\Psi_k+\zeta_k)\rv_2^2 \\
& = \lv \Delta_k \rv_2^2 -2\delta\langle \Delta_k, U_k + \Psi_k + \zeta_k \rangle + \delta^2 \lv U_k + \Psi_k +\zeta_k \rv_2^2.
\end{align*}
Now we take expectation with respect to all sources of randomness conditioned (Brownian motion and the randomness in the choice of $S$) on $x_k$ and $y_0$ (thus $\Delta_k$ is fixed). Recall that conditioned on $\Delta_k$, $\Psi_k$ and $\zeta_k$ are zero mean, thus we get
\begin{align*}
\mathbb{E}_{k}\left[\lv\Delta_{k+1} + V_{k}\rv_2^2\right] & = \lv \Delta_k \rv_2^2 -2\delta \underbrace{\langle \nabla f(y_0) - \nabla f(x_k), y_0 - x_k \rangle}_{=:\Omega_1} + \delta^2 \underbrace{\mathbb{E}_k \left[\lv U_k+\Psi_k+\zeta_k \rv_2^2 \right]}_{=:\Omega_2},
\end{align*}
where $\mathbb{E}_k\left[\cdot\right]$ denotes conditioning on $x_k$ and $y_0$. First we control $\Omega_2$
\begin{align*}
\Omega_2 & = \mathbb{E}_k \left[\lv U_k+\Psi_k+\zeta_k \rv_2^2 \right] \\
& = \mathbb{E}_k \left[\lv \nabla f(y_k) - g(x_k) +\zeta_k \rv_2^2 \right] \\
& = \mathbb{E}_k \left[\lv \nabla f(y_k) - \frac{N}{n}\sum_{i\in S} (\nabla f_i(x_k)-g_k^i)-\sum_{i=1}^n g_k^i +\Psi_k \rv_2^2 \right]  \\
& = \mathbb{E}_k \left[\lv \nabla f(y_k) -\nabla f(x_k)  - \frac{N}{n}\sum_{i\in S} (\nabla f_i(x_k)-g_k^i) +\nabla f(x_k)-\sum_{i=1}^n g_k^i +\Psi_k \rv_2^2 \right]  \\
& = \mathbb{E}_k \left[ \lv \frac{N}{n}\sum_{i\in S}( \nabla f_i(y_0) -\nabla f_i(x_k))
-  \frac{N}{n}\sum_{i\in S}( \nabla f_i(y_0) -g_k^i) + \nabla f(y_0)- \sum_{i=1}^n g_k^i + \Psi_k \rv_2^2 \right],
%& = \mathbb{E}_k \left[ \lv \frac{N}{n}\sum_{i\in S}( \nabla f_i(y_0) -\nabla f_i(x_k)) -  \frac{N}{n}\sum_{i\in S}( \nabla f_i(y_0) -h_k^i)-  \frac{N}{n}\sum_{i\in S}( h_k^i -g_k^i) + \nabla f(y_0)- \sum_{i=1}^n g_k^i + \Psi_k \rv_2^2 \right]
\end{align*}
where in the second equality we used the definition of $U_k$ and $\zeta_k$, while in the third equality we used the definition of $g(x_k)$. By Young's inequality we now have,
\begin{align} \label{eq:omegafirstbound}
\Omega_2  &\le 3 \lv\  \nabla f(y_k) -\nabla f(x_k)  \rv_2^2
+ 3 \mathbb{E}_k \left[\lv     \frac{N}{n}\sum_{i\in S}( \nabla f_i(x_k) -g_k^i) - \nabla f(x_k)+ \sum_{i=1}^n g_k^i  \rv_2^2 \right]
+ 3\mathbb{E}\left[\lv \Psi_k \rv_2^2 \right].
%\\
%&\le 3 \lv\  \nabla f(y_k) -\nabla f(x_k)  \rv_2^2
%+ 3 \frac{N}{n}   \sum_{i=1}^N \lv    \nabla f_i(x_k) -g_k^i \rv_2^2
%+ 3\mathbb{E}\left[\lv \Psi_k \rv_2^2 \right]
\end{align}
Let us define the random variable $\beta^{(i)}=\nabla f_i(x_k) -g_k^i-\frac{n}{N}\left( \nabla f(x_k)+ \sum_{i=1}^n g_k^i\right)$. Observing that $\{\beta^{(i)}\}$ are zero mean (taking expectation over $i$) and independent, we have
\begin{align*}
& \mathbb{E}_k \left[\lv     \frac{N}{n}\sum_{i\in S}( \nabla f_i(x_k) -g_k^i) - \nabla f(x_k)+ \sum_{i=1}^n g_k^i  \rv_2^2 \right] \\
& \!= \frac{N^2}{n^2}  \mathbb{E}_k \left[\lv     \sum_{i\in S} \beta^{(i)} \rv_2^2 \right]  \overset{(i)}{=}  \frac{N^2}{n^2}    \sum_{i\in S}  \mathbb{E}_k \left[\lv    \beta^{(i)} \rv_2^2 \right]
  \overset{(ii)}{=}\! \frac{N^2}{n}  \mathbb{E}_k \left[\lv    \beta \rv_2^2 \right]\\
& \overset{(iii)}{\leq}\frac{N^2}{n}  \mathbb{E}_k  \lv    \nabla f_i(x_k) -g_k^i \rv_2^2  = \frac{N}{n} \sum_{i=1}^N  \lv    \nabla f_i(x_k) -g_k^i \rv_2^2,
\end{align*}
where $(i)$ follows as $\beta^{(i)}$ are zero-mean and independent random variables, $(ii)$ follows by the fact that $\beta^{(i)}$ are identically distributed,  $(iii)$ follows by the fact that for any random variable $R$, $\E{\Vert R-\mathbb E R\Vert_2^2}\leq \E{\Vert R\Vert_2^2}$.

Using the following decomposition $
  \lv    \nabla f_i(x_k) -g_k^i \rv_2^2
   \leq  3   \lv    \nabla f_i(x_k) -\nabla f_i(y_k)\rv_2^2
   +  3\lv  \nabla  f_i(y_k) -h_k^i \rv_2^2
   +  3\lv    h_k^i -g_k^i \rv_2^2
$, we have shown $\Omega_2$ may be bounded as follow
\begin{multline}\label{eq:omeg}
\Omega_2\leq 3 \lv\  \nabla f(y_k) -\nabla f(x_k)  \rv_2^2
+9 \frac{N}{n}   \sum_{i=1}^N   \lv    \nabla f_i(x_k) -\nabla f_i(y_k)\rv_2^2 \\
   +  9 \frac{N}{n}   \sum_{i=1}^N\lv  \nabla  f_i(y_k) -h_k^i \rv_2^2
   +  9 \frac{N}{n}   \sum_{i=1}^N\lv    h_k^i -g_k^i \rv_2^2
   +3\mathbb{E}\left[\lv \Psi_k \rv_2^2 \right].
\end{multline}
\textbf{Step 5:}
We will now bound the different term in Equation \ref{eq:omeg}. First, using a similar technique as \cite{dubey2016variance}, we bound the term $\Vert h_k^i-\nabla f_i(y_k)\Vert^2_2$.  Let $p=1-(1-1/N)^n$ be the probability to chose an index, then
\begin{align*}
 \mathbb{E} \Vert h_i^k-\nabla f_i(y_k)\Vert^2_2 &= \sum_{j=0}^{k-1}   \mathbb{E} [\Vert h_k^i-\nabla f_i(y_k)\Vert^2_2 \vert h_k^i=\nabla f_i(y_j) ]\cdot \mathbb P [h_k^i=\nabla f_i(y_j) ] \\
 &= \sum_{j=0}^{k-1}   \mathbb{E}[\Vert \nabla f_i(y_j)-\nabla f_i(y_k)\Vert^2_2  ] \cdot\mathbb P [h_k^i=\nabla f_i(y_j) ] \\
 &\leq \tilde M^2 \sum_{j=0}^{k-1}   \mathbb{E}[\Vert y_j-y_k\Vert^2_2  ]\cdot \mathbb P [h_k^i=\nabla f_i(y_j) ] \\
  &\leq  \tilde M^2  \sum_{j=0}^{k-1}   \mathbb{E} [\Vert\int_{j\delta}^{k\delta}  \nabla f(y_s)ds -\sqrt{2} (\xi_{k\delta}-\xi_{j\delta}) \Vert_2^2] (1-p)^{k-j-1}p\\
    &\leq \tilde M^2 \sum_{j=0}^{k-1}  [2\delta^2 \mathbb E \Vert \nabla f(y) \Vert ^2+4\delta d (k-j)] (1-p)^{k-j-1}p\\
     &\leq 2 p \tilde M^2  \delta^2 \mathbb E \Vert \nabla f(y) \Vert ^2   \sum_{j=1}^{k} j^2(1-p)^{j-1} +4 p d \tilde M^2  \delta \sum_{j=1}^{k} j(1-p)^{j-1}\\
     &\leq \frac{2\delta^2}{p^2} \tilde M^2  \mathbb E \Vert \nabla f(y) \Vert ^2 +\frac{4d\delta \tilde M^2 }{p} \\
          &\leq \frac{8 d \delta \tilde M^2 N}{n} \left[ \frac{\delta N M}{n}  +1\right],
\end{align*}
where we have use lemma and the bound on  $p\geq \frac{n}{2N}$ from Eq.~(19) of \cite{dubey2016variance}.
Therefore
\[
     9 \frac{N}{n}   \sum_{i=1}^N\lv  \nabla  f_i(y_k) -h_k^i \rv_2^2  \leq  72 \frac{d\delta  N M^2}{n^2}\left[ \frac{\delta M N}{n} +1\right].
\]
Using the $\tilde M$-smoothness of each $f_i$, we have
$
  \lv\nabla f_i(x_k) - \nabla f_i(y_k)  \rv_2^2  \leq \tilde M  \langle \nabla f_i(x_k)-\nabla f_i(y_k),x_k-y_k\rangle
$
and therefore using also the $M$-smoothness of $f$
\[
  3 \lv\  \nabla f(y_k) -\nabla f(x_k)  \rv_2^2+  9 \frac{N}{n} \sum_{i=1}^N  \lv\nabla f_i(x_k) - \nabla f_i(y_k)  \rv_2^2  \leq  3M\left(1+\frac{3}{n}\right) \Omega_1.
\]
In addition, as shown in \eqref{eq:omega4upperbound} that
\[
\mathbb{E}\left[\lv \Psi_k \rv_2^2 \right]\leq \frac{2M^2\delta d}{3}.
\]
Therefore we have proved that
 \begin{align*}
\Omega_2\leq 3 M\left(1+\frac{3}{n}\right) \Omega_1  +  9 \frac{N}{n}   \sum_{i=1}^N \lv    \nabla  g_k^i - h_k^i\rv_2^2
 + 2M^2\delta d+ \frac{72d\delta N M^2}{n^2}\left[ \frac{\delta M N}{n} +1\right].
 \end{align*}
 \textbf{Step 6:} We can combine now the previous bound to first obtain
  \begin{align*}
 \mathbb{E}_{k}\left[\lv\Delta_{k+1} + V_{k}\rv_2^2\right]
  &\leq
  \lv \Delta_k \rv_2^2 -2\delta\left(1-3 \left(1+\frac{3}{n}\right) \delta M\right)\Omega_1 + \frac{  9 \delta^2N}{n}   \sum_{i=1}^N \lv    \nabla  g_k^i - h_k^i\rv_2^2 \\
 &+ 2M^2\delta^3 d+ \frac{72d\delta^3 N^3}{n^2}\left[ \frac{\delta M N}{n} +1\right],
 \end{align*}
 and for any $a>0$
 \begin{align*}
\Ep{k}{\lv \Delta_{k+1} \rv_2^2 } \le
&
 \left(1+ a\right)  \lv
  \Delta_k \rv_2^2  -  2\left(1+ a\right) \delta\left(1-{3}\left(1+\frac{3}{n}\right)\delta M \right)\Omega_1+ \left(1+ a\right)\frac{9 \delta^2 N}{n}   \sum_{i=1}^N \lv    \nabla  g_k^i - h_k^i\rv_2^2\\
  &
  + \left(1+ a\right)\delta^3 \square + \left(1+\frac{1}{a} \right){\delta^4}\triangle,
  \end{align*}
 where we have denoted by $\square =   2M^2 d+\frac{72 d N M^2}{n^2}\left[ \frac{\delta M N}{n} +1\right]$ and $\triangle = \frac{1}{2}\left( L^2d^2+M^3d \right)$.

\textbf{Step 7:} We expand now the first part of $T_k$
 We directly obtain:
\begin{align*}
\Ep{k}{\sum_{i}^N\Vert g_{k+1}^i-h_{k+1}^i\Vert_2^2}
&= \Ep{k}{\sum_{i}^N\Vert g_{k}^i-h_{k}^i\Vert_2^2+\sum_{i\in S}\left(\Vert g_{k+1}^{i}-h_{k+1}^{i}\Vert_2^2-\Vert g_{k}^{i}-h_{k}^{i}]\Vert_2^2\right)}\\
&= \sum_{i}^N\Vert g_{k}^i-h_{k}^i\Vert_2^2 +  \frac{n}{N} \sum_{i=1}^N  \Vert \nabla_i(x_k)- \nabla_i(y_k)\Vert_2^2 -\frac{n}{N}\sum_{i=1}^N [\Vert g_{k}^{i}-h_{k}^{i}\Vert_2^2 \\
&= (1-\frac{n}{N})\sum_{i}^N\Vert g_{k}^i-h_{k}^i\Vert_2^2 +\frac{nM}{N^2} \langle \nabla f(x_k)-\nabla f(y_k),x_k-y_k\rangle .
 \end{align*}
\textbf{Step 8:}
We are now able to upper-bound $T_k$:
 \begin{align*}
 \Ep{k}{T_{k+1}} & \leq \left[  1-\frac{n}{N} +\frac{9(1+a)\delta^2 N}{c n}\right] c \sum_{i}^N\Vert g_{k}^i-h_{k}^i\Vert_2^2 + \left(1+ a\right)  \lv \Delta_k \rv_2^2  \\
 &   -  2\left(1+ a\right) \delta\left(1-{3} (1+\frac{3}{n})\delta M -\frac{c nM} {2(1+a) N^2}\right)\Omega_1 + \left(1+ a\right)\delta^3  \square + \left(1+\frac{1}{a} \right){\delta^4}\triangle.
 \end{align*}
 With the strong-convexity of $f$ we obtain
 \[
 \Omega_1\geq m  \lv \Delta_k \rv_2^2.
 \]
 Then
  \begin{align*}
 \Ep{k}{T_{k+1}} & \leq \left[  1-\frac{n}{N} +\frac{9(1+a)\delta^2 N}{c n}\right] c \sum_{i}^N\Vert g_{k}^i-h_{k}^i\Vert_2^2
 \\
 &+ \left(1+ a\right)  \left(1-2 m  \delta\left(1-{3} (1+\frac{3}{n})\delta M -\frac{c n M}{2\delta(1+a)N^2}\right)\right) \lv \Delta_k \rv_2^2 \\
 &   + \left(1+ a\right)\delta^3 \square + \left(1+\frac{1}{a} \right){\delta^4}\triangle.
 \end{align*}
 \textbf{Step 9:} We will now fix the different values of $c$ and $a$ in order to obtain the final recursive bound on $T_k$. With $c= \frac{12 (1+a)\delta^2N^2}{n^2}$ we obtain
   \begin{align*}
 \Ep{k}{T_{k+1}} & \leq \left[  1-\frac{n}{3N} \right] c \sum_{i}^N\Vert g_{k}^i-h_{k}^i\Vert_2^2
 + \left(1+ a\right)  \left(1-2 m  \delta\left(1-{3} (1+\frac{9}{n})\delta M \right)\right) \lv \Delta_k \rv_2^2 \\
 &   + \left(1+ a\right)\delta^3 \square + \left(1+\frac{1}{a} \right){\delta^4}\triangle.
 \end{align*}
  Assume now that $a=\frac{\alpha}{2(1-\alpha)}$ where $\alpha=2 m  \delta\left(1-{3} (1+\frac{9}{n})\delta M \right)$
    \begin{align*}
 \Ep{k}{T_{k+1}} & \leq \left[  1-\frac{n}{3N} \right] c \sum_{i}^N\Vert g_{k}^i-h_{k}^i\Vert_2^2
 +   \left(1- m  \delta\left(1-{3} (1+\frac{9}{n})\delta M \right)\right) \lv \Delta_k \rv_2^2 \\
 &   + \frac{ \left(1- m  \delta\left(1-{3} (1+\frac{9}{n})\delta M \right)\right)}{ \left(1- 2m  \delta\left(1-{3} (1+\frac{9}{n})\delta M \right)\right)}\delta^3 \square +\frac{ \left(1- m  \delta\left(1-{3} (1+\frac{9}{n})\delta M \right)\right)}{  m  \left(1-{3} (1+\frac{9}{n})\delta M \right)}{\delta^3}\triangle.
 \end{align*}
 Let us assume that $n>9$ and $\delta < 1/12M$, then
    \begin{align*}
 \Ep{k}{T_{k+1}} & \leq \left[  1-\frac{n}{3N} \right] c \sum_{i}^N\Vert g_{k}^i-h_{k}^i\Vert_2^2
 +  \left(1- m  \delta/2 )\right) \lv \Delta_k \rv_2^2 \\
 &   + \frac{1-m\delta/2}{1-m\delta}\delta^3 \square + \frac{2-m\delta}{m}{\delta^3}\triangle.
 \end{align*}
Therefore with further simplification
\[
 \Ep{k}{T_{k+1}}  \leq \left(  1-\rho \right) T_k    +2 \delta^3 \square + \frac{2{\delta^3}\triangle}{m},
\]
where we denote by $\rho=\min\{\frac{n}{3N}, m  \delta/2  \} $.

 \textbf{Step 10:} We are able now to solve this recursion to obtain an upper bound on $T_k$. We obtain a recursive argument,
\[
\E{T_k}\leq (1-\rho)^k \E{T_0}+[2\delta^3 \square + \frac{2{\delta^4}\triangle}{m}
] \sum_{i=0}^{k-1}(1-\rho)^i \leq (1-\rho)^k \E{T_0} +[2\delta^3 \square + \frac{2{\delta^3}\triangle}{m}
]\frac{1- (1-\rho)^k}{\rho},
\]
and
\[
T_0=c\sum_{i=0}^N \Vert \nabla f_i(x_0)- \nabla f_i(y_0)\Vert_2^2 +\Vert x_0-y_0\Vert_2^2\leq \left[\frac{cM^2}{N}+1\right]\Vert x_0-y_0\Vert_2^2\leq \left[\frac{24\delta^2NM^2}{n^2}+1\right] \Vert x_0-y_0\Vert_2^2.
\]
Therefore using that $\E{\Vert x_k-y_k\Vert_2^2} \leq \E{T_k}$, we obtain
\[
\E{\Vert x_k-y_k\Vert_2^2} \leq	(1-\rho)^k\left[\frac{24\delta^2NM^2}{n^2}+1\right]  \mathbb{E}\left[\Vert x_0-y_0\Vert_2^2\right] +\left[\frac{2\delta^3 \square}{\rho} + \frac{2{\delta^3}\triangle}{m \rho}
\right].
\]
We have that $\frac{1}{\rho} = \max\{ \frac{3N}{n}, \frac{2}{m\delta}\}\leq	\frac{3N}{n}+\frac{2}{m\delta}$:
\begin{align*}
\E{\Vert x_k-y_k\Vert_2^2} &\leq	(1-\rho)^k \left[\frac{24\delta^2NM^2}{n^2}+1\right] \mathbb{E}\left[\Vert x_0-y_0\Vert_2^2\right] \\
&+ \left[ \frac{6\delta^3N}{n}+\frac{4\delta^2}{m}\right] \left(L^2 d+72 \frac{d N M^2}{n^2}\left[ \frac{\delta M N}{n} +1\right]\right)
+ \left[ \frac{6\delta^3N}{nm}+\frac{4\delta^2}{m^2}\right]\left( L^2d^2+M^3d \right).
\end{align*}
Using the fact that $x_k$ and $y_k$ are optimally coupled,
%\[
%W^2_2(p^k,p^\ast)\leq 	(1-\rho)^k\left[\frac{32\kappa^2}{3N}+1\right]W_2^2(p_0,p^*)+\frac{4\delta^2}{m} \left(L^2 d+72 \frac{d N M^2}{n^2}\left[ \frac{\delta M N}{n} +1\right]\right) + \frac{4\delta^2}{m^2}\left( L^2d^2+M^3d \right).
%\]
the results follows.
\end{proof}

\section{Control Variates with Underdamped Langevin MCMC}
\label{app:controlvariates}
In this section we will prove Theorem \ref{thm:controlvariatethm} and also include details regarding Algorithm \ref{alg:cvulmcmc} that was omitted in Section \ref{sec:controlvariate}. Throughout this section we will assume that assumptions \ref{ass:decom}-\ref{ass:strongconvexity} holds. Crucially in this section we will \emph{not} assume the Hessian of $f$ to be Lipschitz (\ref{ass:HessianLipschitz}).

Underdamped Langevin Markov Chain Monte Carlo \citep[see, e.g.,][]{cheng2017underdamped} is a sampling algorithm which can be viewed as discretized dynamics of the following It\^{o} stochastic differential equation (SDE):
\begin{align}\label{e:exactlangevindiffusion}
d v_t &= -\gamma v_t dt - u \nabla  f(x_t) dt + (\sqrt{2\gamma u}) dB_t \\
\nonumber d x_t &= v_t dt,
\end{align}
where $(x_t,v_t) \in \mathbb{R}^{2d}$, $f$ is a twice continuously differential function and $B_t$ represents standard Brownian motion in $\mathbb{R}^d$. In the discussion that follows we will always set
\begin{align} \label{eq:gammaudef}
\gamma = 2,\qquad \text{ and,}\qquad u = \frac{1}{M}.
\end{align}
We denote by $p^*$ the unique distribution which satisfies $p^*(x,v) \propto \exp{-(f(x)+\frac{M}{2}\lv v\rv_2^2)}$. It can be shown that $p^*$ is the unique invariant distribution of \eqref{e:exactlangevindiffusion} \citep[see, e.g.,][Proposition 6.1]{pav}. We will choose our intial distribution to always be a Dirac delta distribution $p_0$ centered at $(x_0,v_0) \in \mathbb{R}^d$. Also recall that we set $x_0 = x^* = \argmin_{\alpha \in \mathbb{R}^d} f(\alpha)$, and $v_0 = 0$. We denote by $p_t$ the distribution of $(x_t,v_t)$ driven by the continuous time process \eqref{e:exactlangevindiffusion} with initial conditions $(x_0,v_0)$.%We define the distance to optimum be $\mathcal{D}^2 : = \lv x_0 - x^* \rv_2^2$.

With these definitions and notation in place we can first show that $p_t$ contracts exponentially quickly to $p^*$ measured in $W_2$.
\begin{corollary}[  \cite{cheng2017underdamped}, Corollary 7 ]\label{c:exactconvergence}
Let $p_0$ be a distribution with $(x_0,v_0) \sim p_0$.  Let $q_0$ and $q_t$ be the distributions of  $(x_0,x_0 + v_0)$ and $(x_t,x_t+v_t)$, respectively (i.e., the images of $p_0$ and $p_t$ under the map $g(x,v) = (x,x+v)$). Then
$$W_2(q_t,q^*) \leq e^{-  t/2\kappa} W_2(q_0,q^*).$$
\end{corollary}
The next lemma establishes a relation between the Wasserstein distance between $q$ to $q^*$ and between $p$ to $p^*$.
\begin{lemma}[Sandwich Inequality, \cite{cheng2017underdamped}, Lemma 8]\label{l:sandwich}
The triangle inequality for the Euclidean norm implies that
\begin{equation}\label{e:pqsandwich}
\frac{1}{2}W_2(p_t,p^*) \leq W_2(q_t,q^*) \leq 2W_2(p_t,p^*).
\end{equation}
Thus we also get convergence of $p_t$ to $p^*$:
$$W_2(p_t,p^*) \leq 4e^{- t/2\kappa} W_2(p_0,p^*).$$
%The above is inequality is not used anywhere in this paper.
\end{lemma}
\subsection*{Discretization of the Dynamics}
We will now present a discretization of the dynamics in  \ref{e:exactlangevindiffusion}. A natural discretization to consider is defined by the SDE,
\begin{align}
\label{e:discretelangevindiffusion}
d\vt_t &= -\gamma \vt_t dt -u \nabla \tilde{f}(\xt_0) dt + (\sqrt{2\gamma u}) dB_t \\
d \xt_t &= \vt_s dt, \nonumber
\end{align}
with an initial condition $(\xt_0,\vt_0)\sim \pt_0$. The discrete update differs from \eqref{e:exactlangevindiffusion} by using $\xt_0$ instead of $\xt_t$ in the drift of $\vt_s$. Another difference is that in the drift of \eqref{e:discretelangevindiffusion} we use an unbiased estimator of the gradient at $\xt_0$ given by $\nabla\tilde{f}(\xt_0)$ (defined in \eqref{eq:cvgradupdate}). We will only be analyzing the solutions to \eqref{e:discretelangevindiffusion} for small $t$. Think of an integral solution of \eqref{e:discretelangevindiffusion} as a single step of the discrete chain.

Recall that we denote the distribution of $(x_t,v_t)$ driven by the continuous time process \eqref{e:exactlangevindiffusion} by $p_t$; analogously let the distribution of $(\tilde{x}_t,\tilde{v}_t)$ driven by the discrete time process \eqref{e:discretelangevindiffusion} be denoted by $\tilde{p}_t$. Finally we denote by $\Phi_t$ the operator that maps from $p_0$ to $p_t$:
\begin{align}\label{d:phi}
\Phi_t p_0 = p_t.
\end{align}
Analogously we denote by $\tilde{\Phi}_t$ the operator that maps from $p_0$ to $\tilde{p}_t$:
\begin{align}
\tilde{\Phi}_t p_0 = \tilde{p}_t. \label{d:phitt}
\end{align}
By integrating \eqref{e:discretelangevindiffusion} up to time $\delta M$ (we rescale by $M$ such that results are comparable between the three algorithms) we can derive the distribution of the normal random variables used in Algorithm \ref{alg:cvulmcmc}. Borrowing notation from Section \ref{sec:controlvariate}, recall that our iterates in Algorithm \ref{alg:cvulmcmc} are $(x_k,v_k)_{k=0}^T$. The random vector $Z^{k+1}(x_k,v_k) \in \mathbb{R}^d$, conditioned on $(x_k,v_k)$ has a Gaussian distribution with conditional mean and covariance obtained from the following computations:
\begin{align}
\label{eq:defofZnormal}&\E{v_{k+1}} = v_k e^{-2 \d M} - \frac{1}{2M}(1-e^{-2 \d M}) \nabla \tilde{f}(x_k)\\
\nonumber&\E{x_{k+1}}  = x_k + \frac{1}{2}(1-e^{-2 \d M})v_k - \frac{1}{2M} \left( \d M - \frac{1}{2}\left(1-e^{-2 \d M}\right) \right) \nabla \tilde{f}(x_k)\\
\nonumber&\E{\left(x_{k+1} - \E{x_{k+1}}\right) \left(x_{k+1} - \E{x_{k+1}}\right)^{\top}}= \frac{1}{M } \left[\d M-\frac{1}{4}e^{-4\d M}-\frac{3}{4}+e^{-2\d M}\right] \cdot I_{d\times d}\\
\nonumber&\E{\left(v_{k+1} - \E{v_{k+1}}\right) \left(v_{k+1} - \E{v_{k+1}}\right)^{\top}} = \frac{1}{M}(1-e^{-4 \d M})\cdot I_{d\times d}\\
\nonumber&\E{\left(x_{k+1} - \E{x_{k+1}}\right) \left(v_{k+1} - \E{v_{k+1}}\right)^{\top}}= \frac{1}{2M} \left[1+e^{-4\d M}-2e^{-2\d M}\right] \cdot I_{d \times d}.
\end{align}
A reference to the above calculation is Lemma 11 of \cite{cheng2017underdamped}.
Given this choice of discretization we can bound the discretization error between the solutions to \eqref{e:exactlangevindiffusion} and \eqref{e:discretelangevindiffusion} for small time $\delta M$ (step-size -- $\delta$). Note that if we start from an initial distribution $p_0$, then taking $\ell$ step of the discrete chain with step size $\delta$ maps us to the distribution,
\begin{align} \label{eq:Tstepsofchaindef}
(\tilde{\Phi}_{\delta})^{\ell} p_0 =: p^{(\ell)} \qquad \ell \in \{1,\ldots,T\}.
\end{align}
Corollary \ref{c:exactconvergence} (contraction of the continuous time-process) coupled with the result presented as Theorem \ref{t:discretizationerror} (discretization error bound; stated and proved in Appendix \ref{ss:discreteerrorsection}) will now help us prove Theorem \ref{thm:controlvariatethm}.
\begin{proof}[Proof of Theorem \ref{thm:controlvariatethm}]
For any random variable $(x,v)$ with distribution $p$, let $q$ denote the distribution of the random variables $(x,x+v)$. From Corollary \ref{c:exactconvergence}, we have that for any $i \in \{1,\ldots,T \}$
%$$W_2(\Phi_\d q^{(i)},q^*)\leq e^{-\kappa \delta/2}W_2(q^{(i)}, q^*) \le \left(1- \frac{\kappa\delta}{4}\right)W_2(q^{(i)}, q^*)$$
$$W_2(\Phi_\d q^{(i)},q^*)\leq e^{- m\delta/2}W_2(q^{(i)}, q^*).$$
%the second inequality follows as $\kappa\delta/2 < 1$.
By the discretization error bound in Theorem \ref{t:discretizationerror} and Lemma \ref{l:sandwich}, we get
$$W_2(\Phi_\d q^{(i)}, \Phit_\d q^{(i)})\leq 2W_2(\Phi_\d p^{(i)}, \Phit_\d p^{(i)})\leq M^2\d^2 \sqrt{\frac{16\ke}{5}} + M\delta \sqrt{\frac{32\mathbb{E}_{x\sim p^{(i)}}\lv x - x^* \rv_2^2}{3n}}.$$
By the triangle inequality for $W_2$,
\begin{align}
W_2(q^{(i+1)}, q^*) = W_2(\Phit_\d q^{(i)}, q^*) & \le W_2(\Phi_\d q^{(i)}, \Phit_\d q^{(i)}) + W_2(\Phi_\d q^{(i)},q^*)\\
& \le M^2\d^2 \sqrt{\frac{16\ke}{5}} + M\delta \sqrt{\frac{32\mathbb{E}_{x\sim p^{(i)}}\lv x - x^* \rv_2^2}{3n}}+ e^{-m\delta/2}W_2(q^{(i)}, q^*) \label{e:singlestepdiscreteimprovement}.
\end{align}
Let us define $\eta = e^{-m\delta/ 2}$. Then by applying \eqref{e:singlestepdiscreteimprovement}  $T$ times we have:
\begin{align*}
W_2(q^{(T)}, q^*) & \le \eta^{T}W_2(q^{(0)}, q^*)+ \left(1+ \eta + \ldots
+ \eta^{T-1} \right)\left(M^2\d^2 \sqrt{\frac{16\ke}{5}}+ M\delta \sqrt{\frac{32\mathbb{E}_{x\sim p^{(i)}}\lv x - x^* \rv_2^2}{3n}}\right)\\
& \le 2\eta^{T}W_2(p^{(0)}, p^*)+ \left(\frac{1}{1-\eta}\right)\left[M^2\d^2 \sqrt{\frac{16\ke}{5}}+ M\delta \sqrt{\frac{32\mathbb{E}_{x\sim p^{(i)}}\lv x - x^* \rv_2^2}{3n}}\right],
\end{align*}
where the second step follows by summing the geometric series and by applying the upper bound Lemma \ref{l:sandwich}. By another application of \ref{e:pqsandwich} we get:
\begin{align*}
W_2(p^{(T)}, p^*) & \le \underbrace{4\eta^{T}W_2(p^{(0)}, p^*)}_{=: \Gamma_1}+ \underbrace{\left(\frac{1}{1-\eta}\right)\left[M^2\d^2 \sqrt{\frac{64\ke}{5}}+ M\delta \sqrt{\frac{128\mathbb{E}_{x\sim p^{(i)}}\lv x - x^* \rv_2^2}{3n}}\right]}_{=:\Gamma_2}.
\end{align*}
Observe that
\begin{align*}
1-\eta = 1-e^{-m\delta/2} & \ge \frac{m\delta}{4}.
\end{align*}
This inequality follows as $ m\delta <1$. Note that by Lemma \ref{l:kineticenergyisbounded} we have $\ke \le 26d/m$ and by Lemma \ref{l:varianceboundx} we have $\mathbb{E}_{x \sim p^{(i)}}\lv x-x^* \rv_2^2 \le 10d/m$. Using these bounds we get,%We now bound both terms $\Gamma_1$ and $\Gamma_2$ at a level $\epsilon/2$ to bound the total error $W_2(p^{(n)},p^*)$ at a level $\epsilon$.
\begin{align}
W_2(p^{(T)},p^*) & \le 4 \exp\left(-\frac{m\delta T}{2} \right) W_2(p^{(0)},p^*) + \frac{4M^2 \delta}{m} \sqrt{\frac{1664d }{5m}}+ \frac{4M}{m}\sqrt{\frac{1280d }{3m n}}. \label{eq:wassersteincontractionresult}
\end{align}
In Lemma \ref{lem:initialdistancebound} we establish a bound on $W_2^2(p^{(0)},p^*) \le 2d/m$. This motivates our choice of $T > \frac{1}{m \delta}\log\left( \frac{12 d }{\epsilon m} \right)$,  $\delta = \frac{\epsilon}{M^2} \sqrt{5m^3/(425984\cdot d)}$ and $n = (249M)^2  d /(m^3\epsilon^2)$ which establishes our claim.
\end{proof}
\subsection{Discretization Error Analysis}\label{ss:discreteerrorsection}
In this section we study the solutions of the discrete process \eqref{e:discretelangevindiffusion} up to $t=\delta M$ for some small $\delta$. Here, $\delta$ represents a \emph{single step} of the Langevin MCMC algorithm. In Theorem \ref{t:discretizationerror} we bound the discretization error between the continuous-time process \eqref{e:exactlangevindiffusion} and the discrete process \eqref{e:discretelangevindiffusion} starting from the same initial distribution.
In particular, we bound $W_2 (\Phi_\delta p_0,\Phit_\delta p_0 )$. Recall the definition of $\Phi_t$ and $\Phit_t$ from \eqref{d:phi} and \eqref{d:phitt}. In this section we will assume for now that the kinetic energy (second moment of velocity) is bounded for the continuous-time process,
\begin{equation}
\label{e:energyisbounded}
\forall t\in [0,\d M]\quad \Ep{p_t} {\|v\|_2^2}\leq \ke.
\end{equation}
We derive an explicit bound on $\ke$ (in terms of problem parameters $d,M,m$ etc.) in Lemma \ref{l:kineticenergyisbounded} in Appendix \ref{s:kineticenergybound}.
 We first state a result from \citep{baker2017control} that controls the error between $\nabla\tilde{f}(x_k)$ and $\nabla f(x_k)$.
\begin{lemma}[\cite{baker2017control}, Lemma 1]\label{eq:discreteerrorcontrollemma}Let $(x_k,v_k)$ be the $k^{th}$ iterate of Algorithm \ref{alg:cvulmcmc} with step size $\delta$. Define $\xi_k : = \nabla \tilde{f}(x) - \nabla f(x)$, so that $\xi_k$ measures the noise in the gradient estimate $\nabla \tilde{f}(x)$ and has mean $0$. Then for all $x_k \in \mathbb{R}^d$ and for all $k = 1,\ldots,T$ we have
\begin{align}
\mathbb{E}_{x_k \sim p^{(k)}}\left[\lv \xi_k \rv_2^2 \right] \le \frac{M^2}{n} \mathbb{E}_{x_k \sim p^{(k)}}\left[\lv x_k - x^* \rv_2^2\right].
\end{align}
\end{lemma}
In this section, we will repeatedly use the following inequality:
$$\eu{\int_0^t v_s ds}_2^2  =  \eu{\frac{1}{t}\int_0^t t \cdot v_s ds}_2^2 \leq t\int_0^t \|v_s\|_2^2 ds,$$
which follows from Jensen's inequality using the convexity of $\|\cdot\|_2^2$.
\begin{theorem}\label{t:discretizationerror}
Let $\Phi_t$ and $\Phit_t$ be as defined in \eqref{d:phi} corresponding to the continuous-time and discrete-time processes respectively. Let $p_0$ be any initial distribution and assume that the step size $\delta\le 1/M$. Then the distance between the continuous-time process and the discrete-time process is upper bounded by
$$W_2(\Phi_\d p_0,\Phit_\d p_0)\leq M^2\d^2 \sqrt{\frac{4\ke}{5}} + M\delta \sqrt{\frac{8\mathbb{E}_{x \sim p_0}\left[\lv x - x^* \rv_2^2\right]}{3n}}.$$
\end{theorem}
\begin{proof}[Proof of Theorem \ref{t:discretizationerror}] We will once again use a standard synchronous coupling argument, in which $\Phi_\d p_0$ and $\Phit_\d p_0$ are coupled through the same initial distribution $p_0$ and common Brownian motion $B_t$.

First, we bound the error in velocity. By using the expression for $v_t$ and $\tilde{v}_t$ from Lemma \ref{l:explicitform}, we have
\begin{align*}
\E{\eu{v_s -\vt_s }_2^2 } & \overset{(i)}{=} \mathbb{E} \left[ \left\lv u \int_0^s e^{-2(s-r)}\left( \nabla f(x_r) - \nabla \tilde{f}(x_0) \right)dr\right\rv_2^2\right]\\
& = u^2 \mathbb{E} \left[ \left\lv  \int_0^s e^{-2(s-r)}\left( \nabla f(x_r) - \nabla \tilde{f}(x_0) dr\right)\right\rv_2^2\right]\\
& \overset{(ii)}{\le} s u^2 \int_0^s \mathbb{E} \left[ \left\lv e^{-2(s-r)}\left( \nabla f(x_r) - \nabla \tilde{f}(x_0) \right)\right\rv_2^2\right]dr\\
& \overset{(iii)}{\le} s u^2 \int_0^s \mathbb{E} \left[ \left\lv \left( \nabla f(x_r) - \nabla f(x_0) + \nabla f(x_0) - \nabla \tilde{f}(x_0) \right)\right\rv_2^2\right]dr\\
& \overset{(iv)} {\le} 2s u^2 \int_0^s \mathbb{E} \left[ \left\lv \left( \nabla f(x_r) - \nabla f(x_0) \right)\right\rv_2^2\right]dr + 2s u^2 \int_0^s \mathbb{E} \left[ \left\lv \left( \nabla f(x_0) - \nabla \tilde{f}(x_0) \right)\right\rv_2^2\right]dr\\
& \overset{(v)}{\le} 2s u^2 M^2 \int_0^s \mathbb{E} \left[\left \lv x_r - x_0 \right \rv_2^2\right]dr + \frac{2s^2u^2M^2}{n} \mathbb{E}_{x_0 \sim p_0} \left[\lv x_0 -x^* \rv_2^2 \right]\\
& \overset{(vi)}{=} 2s u^2 M^2 \int_0^s  \mathbb{E} \left[\left \lv \int_0^r v_w dw \right \rv_2^2\right]dr+ \frac{2s^2u^2M^2}{n} \mathbb{E}_{x_0 \sim p_0} \left[\lv x_0 -x^* \rv_2^2 \right]\\
& \overset{(vii)}{\le} 2s u^2 M^2\int_0^s r \left(\int_0^r \mathbb{E}\left[ \lv v_w \rv_2^2\right] dw\right)dr+ \frac{2s^2u^2M^2}{n} \mathbb{E}_{x_0 \sim p_0} \left[\lv x_0 -x^* \rv_2^2 \right]\\
& \overset{(viii)}{\le} 2s u^2 M^2 \ke \int_0^s r \left(\int_0^r dw\right) dr+ \frac{2s^2u^2M^2}{n} \mathbb{E}_{x_0 \sim p_0} \left[\lv x_0 -x^* \rv_2^2 \right]\\
& = \frac{2s^4 u^2 M^2 \ke}{3}+ \frac{2s^2u^2M^2}{n} \mathbb{E}_{x_0 \sim p_0} \left[\lv x_0 -x^* \rv_2^2 \right],
\end{align*}
where $(i)$ follows from the Lemma \ref{l:explicitform} and $v_0=\vt_0$, $(ii)$ follows from application of Jensen's inequality, $(iii)$ follows as $\lvert e^{-4(s-r)}\rvert \le 1$, $(iv)$ follows by Young's inequality, $(v)$ is by application of the $M$-smoothness property of $f(x)$ and by invoking Lemma \ref{eq:discreteerrorcontrollemma} to bounds the second term, $(vi)$ follows from the definition of $x_r$, $(vii)$ follows from Jensen's inequality and $(viii)$ follows by the uniform upper bound on the kinetic energy assumed in \eqref{e:energyisbounded}, and proven in Lemma \ref{l:kineticenergyisbounded}.

This completes the bound for the velocity variable. Next we bound the discretization error in the position variable:
%\begin{align*}
%& \E{\eu{v_s -\vt_s }_2^2 }\\
%=& \E{\eu{\int_0^s (-2 v_r - \frac{1}{L}\nabla f(x_r)) - (-2 \vt_r - \frac{1}{L} \nabla f(\xt_0))dr }_2^2}\\
%\leq& 2\E{\eu{\int_0^s v_r - \vt_r dr}_2^2} + 2 \E{\eu{\int_0^s\frac{1}{L} \nabla f(x_r) - \frac{1}{L} \nabla f(\xt_r) dr}_2^2}\\
%\leq& 2s\int_0^s\E{\eu{ v_r - \vt_r }_2^2}dr + 2s\int_0^s \E{\eu{\frac{1}{L} \nabla f(x_r) - \frac{1}{L} \nabla f(\xt_r) }_2^2} dr\\
%\leq& 2s^2\E{\eu{ v_s - \vt_s }_2^2} + 2s\int_0^s \E{\eu{\frac{1}{L} \nabla f(x_r) - \frac{1}{L} \nabla f(\xt_r) }_2^2} dr
%\end{align*}
%Rearranging the terms, and assuming that $s\leq \delta \leq \frac{1}{2}$, we get
%\begin{align*}
%& \E{\eu{v_s -\vt_s }_2^2 }\\
%\leq& \frac{4s}{L^2}\int_0^s \E{\eu{ \nabla f(x_r) - \nabla f(\xt_r) }_2^2} dr\\
%\leq&  4s \int_0^s \E{\eu{x_r-x_0}_2^2} dr\\
%=& 4s \int_0^s \E{\eu{\int_0^r v_t dt}_2^2} dr\\
%\leq& 4s \int_0^s r \int_0^r \E{\eu{v_t}_2^2} dt dr\\
%\leq & 4 s^4 C
%\end{align*}
%
%Where the third line is by $L$-lipschitz continuous gradients, the fourth line is by \eqref{e:energyisbounded}, the fifth line is by Jensens, and the last line is by the assumption in \eqref{e:energyisbounded}.
\begin{align*}
\E{\eu{x_s - \xt_s}_2^2} & = \E{\eu{\int_0^s (v_r - \vt_r) dr}_2^2}\\
& \leq s\int_0^s \mathbb{E} \left[\lv v_r - \tilde{v}_r \rv_2^2\right]dr\\
& \le s \int_0^s \left(\frac{2r^4u^2M^2\ke}{3} +\frac{2s^2u^2M^2}{n}\mathbb{E}_{x_0 \sim p_0}\left[\lv x_0 - x^* \rv_2^2\right]\right)dr\\
& = \frac{2s^6 u^2 M^2 \ke }{15} + \frac{2s^4u^2M^2}{3n}\mathbb{E}_{x_0 \sim p_0}\left[\lv x_0 - x^* \rv_2^2\right],
\end{align*}
where the first line is by coupling through the initial distribution $p_0$, the second line is by Jensen's inequality and the third inequality uses the preceding bound. Setting $s = M\delta$ and by our choice of $u= 1/M$ we have that the squared Wasserstein distance is bounded as
\begin{align*}
W^2_2(\Phi_\delta p_0,\tilde{\Phi} p_0) \le 2\ke \left(\frac{M^4\delta^4}{3}+ \frac{M^6\delta^6}{15}\right) + \frac{2\mathbb{E}_{x_0 \sim p_0}\left[\lv x_0 - x^* \rv_2^2\right]}{n}\left(M^2\delta^2 + \frac{M^4\delta^4}{3} \right).
\end{align*}
Given our assumption that $\delta$ is chosen to be smaller than $1/M$, this gives the upper bound:
\begin{align*}
W^2_2(\Phi_\delta p_0,\tilde{\Phi} p_0) \le \frac{4\ke M^4\delta^4}{5} + \frac{8M^2\delta^2\mathbb{E}_{x_0 \sim p_0}\left[\lv x_0 - x^* \rv_2^2\right]}{3n}.
\end{align*}
Taking square roots establishes the desired result.
\end{proof}
\subsection{Auxiliary Results}
\label{s:kineticenergybound}
In this section, first we establish an explicit bound on the kinetic energy $\ke$ in \eqref{e:energyisbounded} which is used to control the discretization error at each step.
\begin{lemma}[Kinetic Energy Bound]\label{l:kineticenergyisbounded}
Let $p^{(0)}(x,v) =1_{x=x^*} \cdot 1_{v=0}$--- the Dirac delta distribution at $(x^*,0)$. Further let $p^{(i)}$ be defined as in \eqref{eq:Tstepsofchaindef} for $i=1,\ldots T$, with step size $\d$ and number of iterations $T$ as specified in Theorem \ref{thm:controlvariatethm}. Then for all $i=1,\ldots T$ and for all $t\in[0,\d]$, we have the bound
$$\Ep{(x,v)\sim \Phi_t p^{(i)}}{\|v\|_2^2}\leq \ke,$$
with $\ke = 26d/m$.
\end{lemma}
\begin{proof} We first establish an inequality that provides an upper bound on the kinetic energy for any distribution $p$.
\textbf{Step 1}:  Let $p$ be any distribution over $(x,v)$, and let $q$ be the corresponding distribution over $(x,x+v)$. Let $(x',v')$ be random variables with distribution $p^*$. Further let $\zeta \in \Gamma_{opt}(p,p^*)$ such that,
 $$\mathbb{E}_{\zeta} \left[\lv x-x'\rv_2^2 + \lv (x-x')+(v-v')\rv_2^2 \right] = W_2^2(q,q^*).$$ Then we have,
\begin{align}
\nonumber\Ep{p}{\|v\|_2^2} & = \mathbb{E}_{\zeta}\left[\lv v- v' + v'\rv^2_2\right]\\
\nonumber & \le 2\Ep{p^*}{\|v \|_2^2} + 2  \mathbb{E}_{\zeta}\left[ \lv v- v' \rv^2_2\right]\\
\nonumber & \le 2\Ep{p^*}{\|v \|_2^2} + 4 \mathbb{E}_{\zeta}\left[ \lv x+v - (x'+v') \rv^2_2 + \lv x - x' \rv_2^2 \right]\\
& = 2\Ep{p^*}{\|v \|_2^2} + 4 W_2^2 (q,q^*), \label{e:kineticlessthanwasserstein}
\end{align}
where for the second and the third inequality we have used Young's inequality, while the final line follows by optimality of $\zeta$.

\textbf{Step 2}: We know that $p^*\propto \exp(-(f(x) + \frac{M}{2}\|v\|_2^2))$, so we have $\Ep{p^*}{\|v\|_2^2} = d/M$.

\textbf{Step 3}: For our initial distribution $p_0 (q_0)$ we have the bound
\begin{align*}
W_{2}^{2}(q^{(0)},q^*) & \le 2 \mathbb{E}_{p^*} \left[ \lv v \rv_2^2\right] + 2 \mathbb{E}_{x \sim p^{(0)}, x' \sim p^*} \left[\lv x-x' \rv_2^2 \right]\\
& = \frac{2d}{M} + 2 \mathbb{E}_{x\sim p^*} \left[ \lv x - x^* \rv_2^2\right],\\
& \le \frac{2d}{M} + \frac{2d}{m} \le \frac{4d}{m}.
\end{align*}
%where the first inequality is an application of Young's inequality and the second inequality follows by applying Theorem \ref{t:xvariance}. Combining these we have the bound,
%\begin{align*}
%W_{2}^{2}(q^{(0)},q^*) \le 2d\left(\frac{1}{M}+ \frac{2}{m} \right) \le \frac{6d}{m}.
%\end{align*}
Putting all this together along with \eqref{e:kineticlessthanwasserstein} we have
\begin{align*}
\mathbb{E}_{p^{(0)}}\left[ \lv v \rv_2^2\right] & \le \frac{2d}{M} + \frac{24 d}{m}\le 26\frac{ d}{m}.
\end{align*}
%To analyze the second term we will first derive a lower bound on the normalizing factor,
%\begin{align*}
%e^{\beta f(x^*)} \int e^{-\beta f(x)} dx & = \int e^{-\beta(f(x) - f(x^*))}dx \\
%& \ge \int e^{-\frac{\beta L }{2}(\lv x-x^* \rv_2^2)} dx = \sqrt{\frac{2\pi}{\beta L}}
%\end{align*}
%where the inequality follows from the L-Lipshitz gradients of $f$. Given this bound we have
%\begin{align*}
%\mathbb{E}_{p^*} \left[ \lv x - x^* \rv_2^2\right] & = \frac{1}{\int e^{-\beta f(x) dx}} \left(\int e^{-\beta f(x)} \lv x- x^* \rv_2^2\right)\\
%& = \frac{1}{\int e^{-\beta (f(x)-f(x^*)) dx}}\left(\int e^{-\beta (f(x)-f(x^*))} \lv x- x^* \rv_2^2\right)\\
%& \le \left(\sqrt{\frac{\beta L}{2 \pi}} \right) \left(\int e^{-\frac{\beta m}{2} \lv x-x^* \rv_2^2} \lv x- x^* \rv_2^2\right) \\
%& \le  \left(\sqrt{\frac{\beta L}{2 \pi}} \right) \left( \frac{d}{\beta m} \sqrt{\frac{2 \pi}{\beta m}} \right)\\
%& = d \sqrt{\frac{L}{\beta^2 m^3}}
%\end{align*}
\textbf{Step 4}:
By Corollary \ref{c:exactconvergence}, we know that $\forall t > 0$,
\begin{align*}
W_2^2 (\Phi_t q^{(i)}, q^*)\leq W_2^2 (q^{(i)}, q^*).
\end{align*}
This proves the theorem statement for $i = 0$. We will now prove it for $i > 0$ via induction. We have proved it for the base case $i=0$, let us assume that the result holds for some $\ell \in \{1,\ldots,T \}$. Then by equation Theorem \ref{thm:controlvariatethm} applied upto $\ell$ steps, we know that
$$W_2^2 (q^{(\ell+1)},q^*) = W_2^2 (\Phit_\d q^{(\ell)},q^*)\leq W_2^2(q^{(\ell)},q^*). $$
Thus by \eqref{e:kineticlessthanwasserstein} we have,
$$\Ep{\Phi_{t} p^{(i)}}{\|v\|_2^2}\leq \ke,$$
for all $t>0$ and $i \in \{0,1,\ldots,T \}$.
\end{proof}
Now we provide an upper bound on $\mathbb{E}_{x \sim p^{(i)}}\left[\lv x -x^* \rv_2^2\right]$ that will again be useful in controlling the discretization error.
\begin{lemma}[Variance Bound]\label{l:varianceboundx}Let $p^{(0)}(x,v) =1_{x=x^*} \cdot 1_{v=0}$--- the Dirac delta distribution at $(x^*,0)$. Further let $p^{(i)}$ be defined as in \eqref{eq:Tstepsofchaindef} for $i=1,\ldots T$, with step size $\d$ and number of iterations $T$ as specified in Theorem \ref{thm:controlvariatethm}. Then for all $i=1,\ldots T$ and for all $t\in[0,\d]$, we have the bound
$$\Ep{x\sim \Phi_t p^{(i)}}{\|x-x^*\|_2^2}\leq \frac{10d}{m}.$$
\end{lemma}
\begin{proof}We first establish an inequality that provides an upper bound on the kinetic energy for any distribution $p$.
\textbf{Step 1}:  Let $p$ be any distribution over $(x,v)$, and let $q$ be the corresponding distribution over $(x,x+v)$. Let $(x',v')$ be random variables with distribution $p^*$. Further let $\zeta \in \Gamma_{opt}(p,p^*)$ such that,
 $$\mathbb{E}_{\zeta} \left[\lv x-x'\rv_2^2 + \lv (x-x')+(v-v')\rv_2^2 \right] = W_2^2(q,q^*).$$ Then we have,
\begin{align}
\nonumber\Ep{x\sim p}{\|x-x^*\|_2^2} & = \mathbb{E}_{\zeta}\left[\lv x- x' + x' - x^*\rv^2_2\right]\\
\nonumber & \le 2\Ep{x' \sim p^*}{\|x'-x^* \|_2^2} + 2  \mathbb{E}_{\zeta}\left[ \lv x- x' \rv^2_2\right]\\
%\nonumber & \le 2\Ep{x'\sim p^*}{\|x'-x^* \|_2^2} + 4 \mathbb{E}_{\zeta}\left[ \lv x+v - (x'+v') \rv^2_2 + \lv x - x' \rv_2^2 \right]\\
& = 2\Ep{x \sim p^*}{\|x-x^* \|_2^2} + 2 W_2^2 (q,q^*), \label{e:kineticlessthanwasserstein2}
\end{align}
where for the second and the third inequality we have used Young's inequality, while the final line follows by optimality of $\zeta$.

\textbf{Step 2}: We know by Theorem \ref{t:xvariance} that $\Ep{x \sim p^*}{\|x-x^*\|_2^2} \le  d/m$.

\textbf{Step 3}: For our initial distribution $p_0 (q_0)$ we have the bound
\begin{align*}
W_{2}^{2}(q_0,q^*) & \le 2 \mathbb{E}_{p^*} \left[ \lv v \rv_2^2\right] + 2 \mathbb{E}_{x \sim p^{(0)}, x' \sim p^*} \left[\lv x-x' \rv_2^2 \right]\\
& = \frac{2d}{M} + 2 \mathbb{E}_{x\sim p^*} \left[ \lv x - x^* \rv_2^2\right],\\
& \le \frac{2d}{M} + \frac{2d}{m} \le \frac{4d}{m}.
\end{align*}
where the first inequality is an application of Young's inequality, the equality in the second line follows as $p^*(v) \propto \exp(-M\lv v \rv_2^2/2)$ and the second inequality follows by again applying the bound from Theorem \ref{t:xvariance}. Combining these we have the bound,
%\begin{align*}
%W_{2}^{2}(q^{(0)},q^*) \le 2d\left(\frac{1}{M}+ \frac{2}{m} \right) \le \frac{6d}{m}.
%\end{align*}
Putting all this together along with \eqref{e:kineticlessthanwasserstein2} we have
\begin{align*}
\mathbb{E}_{x\sim p_0}\left[ \lv x-x^* \rv_2^2\right] & \le \frac{2d}{m} + \frac{8 d}{m}\le 10\frac{ d}{m}.
\end{align*}
%To analyze the second term we will first derive a lower bound on the normalizing factor,
%\begin{align*}
%e^{\beta f(x^*)} \int e^{-\beta f(x)} dx & = \int e^{-\beta(f(x) - f(x^*))}dx \\
%& \ge \int e^{-\frac{\beta L }{2}(\lv x-x^* \rv_2^2)} dx = \sqrt{\frac{2\pi}{\beta L}}
%\end{align*}
%where the inequality follows from the L-Lipshitz gradients of $f$. Given this bound we have
%\begin{align*}
%\mathbb{E}_{p^*} \left[ \lv x - x^* \rv_2^2\right] & = \frac{1}{\int e^{-\beta f(x) dx}} \left(\int e^{-\beta f(x)} \lv x- x^* \rv_2^2\right)\\
%& = \frac{1}{\int e^{-\beta (f(x)-f(x^*)) dx}}\left(\int e^{-\beta (f(x)-f(x^*))} \lv x- x^* \rv_2^2\right)\\
%& \le \left(\sqrt{\frac{\beta L}{2 \pi}} \right) \left(\int e^{-\frac{\beta m}{2} \lv x-x^* \rv_2^2} \lv x- x^* \rv_2^2\right) \\
%& \le  \left(\sqrt{\frac{\beta L}{2 \pi}} \right) \left( \frac{d}{\beta m} \sqrt{\frac{2 \pi}{\beta m}} \right)\\
%& = d \sqrt{\frac{L}{\beta^2 m^3}}
%\end{align*}
\textbf{Step 4}:
By Corollary \ref{c:exactconvergence}, we know that $\forall t > 0$,
\begin{align*}
W_2^2 (\Phi_t q^{(i)}, q^*)\leq W_2^2 (q^{(i)}, q^*).
\end{align*}
This proves the theorem statement for $i = 0$. We will now prove it for $i > 0$ via induction. We have proved it for the base case $i=0$, let us assume that the result holds for some $\ell \in \{1,\ldots,T \}$. Then by equation Theorem \ref{thm:controlvariatethm} applied upto $\ell$ steps, we know that
$$W_2^2 (q^{(\ell+1)},q^*) = W_2^2 (\Phit_\d q^{(\ell)},q^*)\leq W_2^2(q^{(\ell)},q^*). $$
Thus by \eqref{e:kineticlessthanwasserstein2} we have,
$$\Ep{\Phi_{t} p^{(\ell)}}{\|v\|_2^2}\leq \frac{10d}{m},$$
for all $t>0$ and $\ell \in \{0,1,\ldots,T \}$.
\end{proof}
Next we prove that the distance of the initial distribution $p^{(0)}$ to the optimum distribution $p^*$ is bounded.
\begin{lemma}\label{lem:initialdistancebound} Let $p^{(0)}(x,v) =1_{x=x^*} \cdot 1_{v=0}$--- the Dirac delta distribution at $(x^{*},0)$. Then
\begin{align*}
W_2^2(p^{(0)},p^*) \le 2\frac{d}{m}.
\end{align*}
\end{lemma}
\begin{proof} As $p^{(0)}(x,v)$ is a delta distribution, there is only one valid coupling between $p^{(0)}$ and $p^*$. Thus we have
\begin{align*}
W_2^2(p^{(0)},p^*) & = \mathbb{E}_{(x,v) \sim p^*}\left[\lv x-x^{*}\rv_2^2 + \lv v \rv_2^2 \right] \\
%& =\mathbb{E}_{(x,v) \sim p^*}\left[\lv x-x^* + x^* - x^{(0)}\rv_2^2 + \lv v \rv_2^2 \right] \\
& \le \mathbb{E}_{x \sim p^*(x)}\left[\lv x-x^*\rv_2^2 \right] + \mathbb{E}_{v \sim p^*(v)}\left[\lv v \rv_2^2 \right].
\end{align*}
Note that $p^*(v) \propto \exp(-M\lv v\rv_2^2/2)$, therefore $\mathbb{E}_{v \sim p^*(v)}\left[\lv v \rv_2^2 \right] = d/M$. By invoking Theorem \ref{t:xvariance} the first term $\mathbb{E}_{x \sim p^*(x)}\left[\lv x-x^*\rv_2^2 \right]$ is bounded by $d/m$. Putting this together we have,
\begin{align*}
W_2^2(p^{(0)},p^*) & \le \frac{d}{m} + \frac{d}{M} \le 2\frac{d}{m}.
\end{align*}
\end{proof}
Next we calculate integral representations of the solutions to the continuous-time process \eqref{e:exactlangevindiffusion} and the discrete-time process \eqref{e:discretelangevindiffusion}.
\begin{lemma}\label{l:explicitform}
The solution $(x_t,v_t)$ to the underdamped Langevin diffusion \eqref{e:exactlangevindiffusion} is
\begin{align*}
\numberthis \label{e:vdynamics}
v_t &= v_0 e^{-\gamma t} - u \left(\int_0^t e^{-\gamma(t-s)} \nabla f(x_s) ds \right) + \sqrt{2\gamma u} \int_0^t e^{-\gamma (t-s)} dB_s\\
x_t &= x_0 + \int_0^t v_s ds.
\end{align*}
The solution $(\tilde{x}_t,\tilde{v}_t)$ of the discrete underdamped Langevin diffusion \eqref{e:discretelangevindiffusion} is
\begin{align*}
\numberthis \label{e:vtildedynamics}
\vt_t &= \vt_0 e^{-\gamma t} - u \left(\int_0^t e^{-\gamma(t-s)} \nabla \tilde{f}(\xt_0) ds \right) + \sqrt{2\gamma u} \int_0^t e^{-\gamma (t-s)} dB_s\\
\xt_t &= \xt_0 + \int_0^t \vt_s ds.
\end{align*}
\end{lemma}
\begin{proof}%[Proof of Lemma \ref{l:explicitform}]
It can be easily verified that the above expressions have the correct initial values $(x_0,v_0)$ and $(\xt_0,\vt_0)$. By taking derivatives, one also verifies that they satisfy the differential equations in \eqref{e:exactlangevindiffusion} and \eqref{e:discretelangevindiffusion}.
\end{proof}

\section{Technical Results}
We state this Theorem from \cite{durmus2016sampling} used in the proof of Lemma \ref{l:kineticenergyisbounded}.
\begin{theorem}[\cite{durmus2016sampling}, Theorem 1] \label{t:xvariance} For all $t\ge 0$ and $x\in \mathbb{R}^d$,
\begin{align*}
\mathbb{E}_{x\sim p^*}\left[ \lVert x - x^* \rVert_2^2 \right] \le \frac{d}{m}.
\end{align*}
\end{theorem}
%
%We present a useful lemma from \citep{dalalyan2017user} that we will use in the proof of Theorem \ref{thm:controlvariatethm}.
%
%\begin{lemma}[Lemma 7 in \citep{dalalyan2017user}] \label{l:dalalyanuseful} Let $A$, $B$ and $C$ be given non-negative numbers such that $A \in \{0,1\}$. Assume that the sequence of non-negative numbers $\{x_k\}_{k \in \mathbb{N}}$ satisfies the recursive inequality
%\begin{align*}
%x_{k+1}^2 \le \left[A\cdot x_k + C\right]^2 + B^2
%\end{align*}
%for every integer $k \ge 0$. Then
%\begin{align}
%x_k \le A^{k}x_0 + \frac{C}{1-A} + \frac{B^2}{C + \sqrt{(1-A^2)}B}
%\end{align}
%for all integers $k \ge 0$.
%
%\end{lemma}

\end{document}